\documentclass[sigconf]{aamas} 
\usepackage{balance} % for balancing columns on the final page

\usepackage{algorithm}
\usepackage{algorithmic}
\usepackage{subfig}
\usepackage{amsfonts, amsmath, amsthm}
\usepackage{dsfont}
\newtheorem{theorem}{Theorem}
\usepackage{booktabs}
\usepackage{multirow}

\setcopyright{none}
\renewcommand\footnotetextcopyrightpermission[1]{} % removes footnote with conference information in first column
\setcopyright{ifaamas}
\acmConference[AAMAS '24]{Proc.\@ of the 23rd International Conference
on Autonomous Agents and Multiagent Systems (AAMAS 2024)}{May 6 -- 10, 2024}
{Auckland, New Zealand}{N.~Alechina, V.~Dignum, M.~Dastani, J.S.~Sichman (eds.)}
\copyrightyear{2024}
\acmYear{2024}
\acmDOI{}
\acmPrice{}
\acmISBN{}

\acmSubmissionID{401}

\title[AAMAS-2024 Formatting Instructions]{Boosting Continuous Control with Consistency Policy}

\author{Yuhui Chen}
\affiliation{
  \institution{Institute of Automation, Chinese Academy of Sciences}
  \city{BEIJING}
  \country{CHINA}}
\affiliation{
  \institution{School of Artificial Intelligence, University of Chinese Academy of Sciences}
  \city{BEIJING}
  \country{CHINA}}
\email{chenyuhui2022@ia.ac.cn}

\author{Haoran Li}
\affiliation{
  \institution{Institute of Automation, Chinese Academy of Sciences}
  \city{BEIJING}
  \country{CHINA}}
\affiliation{
  \institution{School of Artificial Intelligence, University of Chinese Academy of Sciences}
  \city{BEIJING}
  \country{CHINA}}
\email{lihaoran2015@ia.ac.cn}

\author{Dongbin Zhao}
\affiliation{
  \institution{Institute of Automation, Chinese Academy of Sciences}
  \city{BEIJING}
  \country{CHINA}}
\affiliation{
  \institution{School of Artificial Intelligence, University of Chinese Academy of Sciences}
  \city{BEIJING}
  \country{CHINA}}
\email{dongbin.zhao@ia.ac.cn}

\begin{abstract}
Due to its training stability and strong expression, the diffusion model has attracted considerable attention in offline reinforcement learning. However, several challenges have also come with it: 1) The demand for a large number of diffusion steps makes the diffusion-model-based methods time inefficient and limits their applications in real-time control; 2) How to achieve policy improvement with accurate guidance for diffusion model-based policy is still an open problem. Inspired by the consistency model, we propose a novel time-efficiency method named Consistency Policy with Q-Learning (CPQL), which derives action from noise by a single step. By establishing a mapping from the reverse diffusion trajectories to the desired policy, we simultaneously address the issues of time efficiency and inaccurate guidance when updating diffusion model-based policy with the learned Q-function. We demonstrate that CPQL can achieve policy improvement with accurate guidance for offline reinforcement learning, and can be seamlessly extended for online RL tasks. Experimental results indicate that CPQL achieves new state-of-the-art performance on 11 offline and 21 online tasks, significantly improving inference speed by nearly 45 times compared to Diffusion-QL. Code is available at https://github.com/cccedric/cpql.
\end{abstract}

\newcommand{\BibTeX}{\rm B\kern-.05em{\sc i\kern-.025em b}\kern-.08em\TeX}

\begin{document}

%%% The following commands remove the headers in your paper. For final 
%%% papers, these will be inserted during the pagination process.

\pagestyle{fancy}
\fancyhead{}

%%% The next command prints the information defined in the preamble.

\maketitle 

%%%%%%%%%%%%%%%%%%%%%%%%%%%%%%%%%%%%%%%%%%%%%%%%%%%%%%%%%%%%%%%%%%%%%%%%

\section{Introduction}
\noindent Offline reinforcement learning (RL) offers a promising approach for learning policies from pre-collected datasets to solve sequential decision-making tasks, but it requires conservative behaviors to alleviate value overestimation \cite{lange2012batch, FangZGZ22}. The policy representation needs to be powerful enough to cover the diverse behaviors, thereby alleviating the value function overestimation caused by querying out-of-distribution (OOD) actions \cite{fujimoto2019off, levine2020offline}. Traditional unimodal policy struggles to model multi-modal behaviors in datasets, leading to the failure of policy constraints \cite{wang2022diffusion, chen2022offline, pearce2023imitating}.

The diffusion model \cite{ho2020denoising} showcases remarkable attributes in terms of high training stability and its ability to provide strong distributional representations, achieving impressive outcomes in the generation of high-quality image samples endowed with diverse features. In light of the imperative demand for strong expression ability, the diffusion model is introduced to offline RL for imitating behaviors \cite{hsiang2023diffusion, chi2023visuomotor, li2023crossway}, trajectories modeling \cite{janner2022planning, ajay2022conditional, rempo2023trace}, and building expected policy \cite{wang2022diffusion, lu2023contrastive, zhong2023guided}. Since the diffusion model can cover diverse behaviors \cite{yang2023policy} and enhance the KL divergence constraint between the behavior policy and the expected policy \cite{wang2022diffusion}, these methods achieve impressive performance by alleviating value overestimation with reducing OOD actions.

However, the adoption of the diffusion model also presents two prominent challenges. Firstly, the training and inference of the diffusion-model-based policy is time inefficient. This is because the inherent sampling process of the diffusion model relies on a Markov chain for a large number of steps (e.g., 1,000 steps) to capture intricate dependencies in the data. This inadequacy translates into protracted training durations and sluggish real-time inference capabilities. The ramification of this time inefficiency severely hinders practical utility in real-time decision-making domains such as robot control. Secondly, it is hard to improve the parameterized policy by diffusion model accurately under the actor-critic framework. Due to the absence of data from the optimal behavior, updating diffusion model-based policy needs additional guidance to achieve better behavior. Though the Q-value is employed to guide the reverse diffusion process \cite{wang2022diffusion, kang2023efficient}, it is theoretically impossible to access the desired policy since the Q-value is inaccurate for intermediate diffusion actions. How to achieve accurate guidance for the expected policy is still an open problem \cite{lu2023contrastive}. 

Instead of a large number of reverse diffusion steps, the consistency model \cite{song2023consistency} is based on probability flow ordinary differential equation (ODE) and achieves the one-step generation process. Inspired by the consistency model, we propose a novel approach called Consistency Policy with Q-learning (CPQL) that can generate action directly from noise in a single step. Due to this one-step generation, CPQL significantly outperforms previous diffusion-model-based methods in training and inference speed. By establishing a mapping from the reverse diffusion trajectories to the desired policy, CPQL avoids explicit access to inaccurate guidance functions in multi-step diffusion processes, and we theoretically prove that it can achieve policy improvement with accurate guidance for offline RL. 
% We also propose a surrogate loss of consistency loss for stabilizing policy training. 
Moreover, with the inherent sampling randomness of the stochastic generation process, CPQL can be seamlessly extended for online RL tasks without relying on additional exploration strategies.

In summary, our contribution involves three main aspects. Firstly, we propose a novel time-efficient method named CPQL and improve training and inference speeds by nearly 15x and 45x compared to Diffusion-QL \cite{wang2022diffusion}, while also improving the performance. Secondly, we conduct a theoretical analysis to demonstrate that the CPQL is capable of achieving policy improvement with accurate guidance and propose an empirical loss to replace consistency loss for stabilizing consistency policy training. Finally, CPQL can seamlessly extend to online RL tasks. As experimented, CPQL achieves state-of-the-art (SOTA) performance on 11 offline and 21 online tasks.

%%%%%%%%%%%%%%%%%%%%%%%%%%%%%%%%%%%%%%%%%%%%%%%%%%%%%%%%%%%%%%%%%%%%%%%%

\section{Related Work}
\paragraph{Offline RL}
The main problem offline RL faces is the overestimation of the value function caused by accessing OOD actions. The previous offline RL methods broadly are categorized as policy constraints \cite{fujimoto2019off, fujimoto2021minimalist, wu2019behavior}, value function regularization \cite{AlgaeDICE2019Nachum, kumar2020conservative}, or in sample learning \cite{kostrikov2021offline, ma2021offline, xiao2023sample, xu2023offline}. Our work belongs to policy constraints. Different from the previous method based on unimodal distribution, we propose a sampling efficient policy based on the consistency model, which is better suited to constrain the policy to meet the multi-modal characteristics of the offline dataset.

\paragraph{Diffusion Models for Imitation Learning}
When expert trajectories are accessible, imitation learning is a powerful method for building the expected policy. The diffusion model is employed to tackle diverse behaviors in the expert dataset \cite{pearce2023imitating, chi2023diffusion, li2023crossway}. When the reward for the typical task is sparse or inaccessible, goal-conditioned imitation learning is an alternative solution, and the diffusion model can be used to build the goal-conditioned policy \cite{reuss2023goal}. Moreover, since the diffusion model has a strong text-to-image ability, it is also used to generate the behavior goal with the language as the input \cite{kapelyukh2023dallebot, gao2023can}.

\paragraph{Diffusion Models for RL}
The diffusion model has been widely used in RL, especially offline RL tasks. The diffusion model can be used to model trajectories \cite{janner2022planning, ajay2022conditional}, build the world model \cite{brehmer2023egdi}, augment the dataset \cite{lu2023synthetic} and estimate the action distribution conditioned on the state. This distribution can be either a behavioral policy distribution in the dataset \cite{chen2022offline, pearce2023imitating, hansen2023idql} or an expected policy distribution \cite{wang2022diffusion, kang2023efficient, lu2023contrastive}. Difffusion-QL \cite{wang2022diffusion} and EDP \cite{kang2023efficient} are similar to our work. There are several differences between our work and these methods. First and foremost, they did not clearly point out what distribution the diffusion model fits, while our work starts with modeling the solution of the constrained policy search problem and proposes the consistency policy. Second, these methods are damaged by inaccurate guiding during the reverse diffusion process. Finally, a large number of reverse diffusion steps is necessary for these methods and harm real-time for robot control. QGPO \cite{lu2023contrastive} also focuses on solving the problem of diffusion model fitting the optimal solution. Our work is different in that we solve this problem from the perspective of establishing ODE trajectories mapping rather than the perspective of accurate guidance function during the diffusion process. In addition to offline tasks, DIPO \cite{yang2023policy} is first proposed to use the diffusion model to solve online RL problems. This work proposes the action gradient to update the actions in the replay buffer and uses the diffusion model to fit the updated action distribution. Our approach directly updates the consistency policy using the gradient of the value function. 

%%%%%%%%%%%%%%%%%%%%%%%%%%%%%%%%%%%%%%%%%%%%%%%%%%%%%%%%%%%%%%%%%%%%%%%%

\section{Diffusion policy for offline RL}
\subsection{Offline RL}
\label{sec: offline rl}
A decision-making problem in RL is usually represented by a Markov Decision Process (MDP), defined as a tuple: $\{S, A, P, r, \gamma\}$. $S$ and $A$ represent the state and action spaces, respectively. $P(s_{t+1}|s_t,a_t)$ represents the transition probability from state $s_t$ to next state $s_{t+1}$ after taking the action $a_t$, and $r(s_t,a_t,s_{t+1})$ represents the corresponding reward. $ \gamma$ is the discount factor. A policy $\pi(a_t|s_t)$ describes how an agent interacts with the environment. And we use subscripts $t\in \{1,\cdots,T\}$ to denote trajectory timesteps.

\paragraph{Policy Constraint for Offline Learning}
Without interactions with the environment, agents under the offline RL scenario learn an expected policy entirely from a previously collected static dataset of transitions denoted as $\mathcal{D}_{offline}=\{(s_t,a_t,r_t,s_{t+1})\}$ by a behavior policy $\mu(a_t|s_t)$. In order to alleviate the value overestimation caused by visiting OOD actions \cite{KumarFSTL19}, offline RL methods \cite{levine2020offline} often constrain the learned policy to the behavior policy $\mu$. The goal is to find the desired policy around behavior policy that can achieve the following objective:
\begin{equation}
    \begin{aligned}
    \mathcal{J}(\pi)=\mathbb{E}_{s_t\sim \mathcal{D}_{offline}}[\mathbb{E}_{a_t\sim\pi(\cdot|s_t)}[Q_{\upsilon}(s_t,a_t)] \\ 
    -\lambda D_{KL}(\pi(\cdot|s_t)||\mu(\cdot|s_t))]
    \end{aligned}
    \label{eq:batch constrained RL}
\end{equation}
\noindent where $Q_{\upsilon}(s_t, a_t)$ is a learned Q-function of the current policy $\pi$, and $\lambda$ determines the relative importance of the KL divergence against the value function, controlling the balance between exploitation and policy constraints. 

\paragraph{Policy Improvement with Constraint Policy Search}
In the following sections, we omit subscripts $t$ to simplify the demonstrations if they are unnecessary. With the optimization objective claimed in Eq. (\ref{eq:batch constrained RL}), we have the closed form of the solution $\pi^*$ \cite{chen2022offline, awr2019peng, PetersMA10}:
\begin{equation}
\begin{aligned}
\label{eq:optimal policy}
    \pi^*(a|s)\propto\mu(a|s)\exp(\frac{1}{\lambda} Q_{\upsilon}(s,a))
\end{aligned}
\end{equation}

The aforementioned solution is frequently employed to tackle offline RL problems. These methods typically explicitly estimate the probability distribution of the behavioral policy to compute the analytical form of Eq. (\ref{eq:optimal policy}). However, estimating the probability distribution of the behavioral policy becomes exceedingly challenging when the policy is complex. Therefore, diffusion models are often utilized to model this behavioral policy, which is straightforward due to the availability of a dataset for the behavioral policy. Consequently, the aforementioned diffusion model-based approaches are incapable of estimating the analytical form of Eq. (\ref{eq:optimal policy}), necessitating additional importance sampling for the computation of the policy corresponding to Eq. (\ref{eq:optimal policy}) \cite{Wonjoon2022arq, hansen2023idql, chen2022offline}. Furthermore, in these methods, it is usually possible to only estimate the Q-value of the behavioral policy; thus, such methods are also referred to as one-step bootstrapping. In this paper, we will employ a diffusion model to directly model $\pi^*$, thereby avoiding reliance on Q-value and resampling techniques during the inference process, and achieving improved performance by implementing multi-step bootstrapping.

%%%%%%%%%%%%%%%%%%%%%%%%%%%%%%%%%%%%%%%%%%%%%%%%%%%%%%%%%%%%%%%%%%%%%%%%

\subsection{Diffusion Policy}
In this section, we define the expected policy of Eq.(\ref{eq:optimal policy}) as diffusion policy and describe how to use the diffusion model to build it. Before formally introducing this process, we first declare that there are two different types of timesteps in the following sections. Rather than using subscripts $t$ denoting the trajectory timesteps, we use superscripts $k\in [0, K]$ to denote the diffusion timesteps. 

The diffusion model starts by diffusing $p_{data}(a)$ with a stochastic differential equation (SDE): 
\begin{equation}
    da^k=\mu(a^k, k)dk+\sigma(k)dw^k
\end{equation}

\noindent where $\mu(\cdot,\cdot)$ and $\sigma (\cdot)$ are the drift and diffusion coefficients respectively, and $\{w^k\}_{k\in[0,K]}$ denotes the standard Brownian motion. Starting from $x^K$, the diffusion model aims to recover the original data $x^0$ by solving a reverse process from $K$ to $0$ with the Probability Flow (PF) ODE \cite{song2020score}:
\begin{equation}
    da^k=[\mu(a^k,k)-\frac{1}{2}\sigma(k)^2\nabla \log p^k(a^k)]dk
\end{equation}
\noindent where the only unknown term is the score function $\nabla \log p^k(a^k)$ at each timestep. Thus, the diffusion model trains a neural network parameterized by $\phi$ to estimate the score function: $s_{\phi}(a^k,k)\approx \nabla \log p^k(a^k)$. By setting $\mu(a^k,k)=0$ and $\sigma(k)=\sqrt{2}$, we can obtain an empirical estimate of the PF ODE:
\begin{equation}
    \frac{da^k}{dk}=-ks_{\phi}(a^k,k)
    \label{eq:f}
\end{equation}

When the diffusion model is used to represent the expected policy directly, the challenge lies in estimating the score function because it's impossible to obtain data from the optimal policy straightforwardly, making the estimation intractable. Nevertheless, we can use the diffusion process through reference policy with guidance to estimate the score function \cite{dhariwal2021diffusion}. With Eq. (\ref{eq:optimal policy}), we can derive the following score function with guidance for the solution of Eq. (\ref{eq:optimal policy}):
\begin{equation}
    \label{eq:optimal score}
    \nabla\log\pi^*(a^k|s) = \nabla\log\mu(a^k|s)+\frac{1}{\lambda}\nabla Q_{\upsilon}(s,a^k)
\end{equation}

Note that we have to calculate the Q-function for each diffusion timestep. However, we only have the learned Q-value at diffusion timestep $0$, and it is challenging to accurately estimate the Q-value for each diffusion step. Consequently, achieving the correct policy through accurate diffusion guidance is nontrivial \cite{lu2023contrastive}.

%%%%%%%%%%%%%%%%%%%%%%%%%%%%%%%%%%%%%%%%%%%%%%%%%%%%%%%%%%%%%%%%%%%%%%%%

\section{Consistency Policy with Q-Learning}
\subsection{Consistency Policy}
The previous section notes that modeling the solution using the classifier-guided diffusion process is quite challenging due to the inaccuracy of guidance. 
In this section, we will show how the consistency policy can avoid this problem and achieve policy improvement with accurate guidance. Drawing inspiration from the consistency model, we could directly map the ODE trajectories to the policy in the inverse diffusion process. Following the consistency model \cite{song2023consistency}, we define the consistency policy:
\begin{equation}
    \begin{aligned}
    \pi_{\theta}(a|s) &\equiv f_{\theta}(a^k,k|s) \\
                          &= c_{skip}(k)a^k+c_{out}(k)F_{\theta}(a^k,k|s) \\
    \end{aligned}
\label{eq:sample}
\end{equation}

\begin{figure}[htbp]
\centering
\includegraphics[width=8cm]{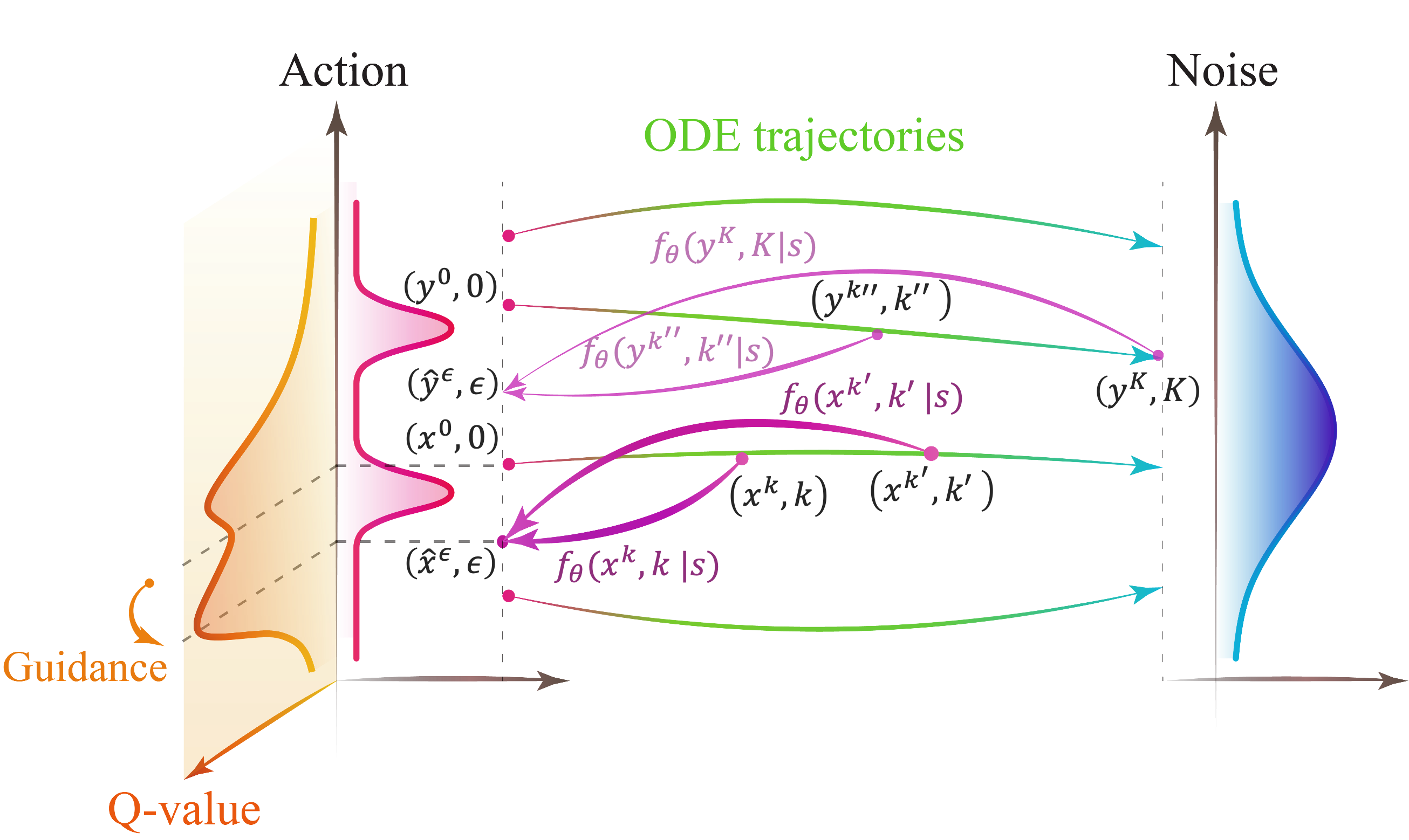}
\caption{Given an ODE that smoothly converts from actions of the reference policy (e.g., $x^0, y^0\in A$) to Gaussian noises, the consistency policy $f_{\theta}$ maps any point (e.g., $x^{k}, x^{k'}, y^{k''}, y^K$) on the PF ODE trajectory to the desired actions (e.g., $\hat{x}^{\epsilon}, \hat{y}^{\epsilon}$). Since consistency policy generates the actions from the noise by one step, it reduces an enormous amount of time for policy training and inference.}
\label{fig:cmql}
\vspace{-0.5em}
\end{figure}

\noindent where $a^k \sim \mathcal{N}(0, kI)$. The $F_{\theta}(a^k,k|s)$ is a trainable network that takes the state $s_t$ as an condition and outputs an action of the same dimensionality as the input $a^k$. $c_{skip}$ and $c_{out}$ are differentiable functions such that $c_{skip}(\epsilon)=1$, and $c_{out}(\epsilon)=0$ to ensure the consistency policy is differentiable at $k=\epsilon$ if $F_{\theta}(a^k,k|s)$ is differentiable, which is critical for training process described later. We stop solving the reverse process at $k=\epsilon$, where $\epsilon$ is a small positive constant to avoid numerical instability. The sampled action $\hat{a}^{\epsilon} \sim \pi_{\theta}(a|s)$ of the consistency policy is used for controlling the agent. The relationship between the forward diffusion process represented by the ODE trajectories and consistency policy is shown in Figure \ref{fig:cmql}.

\subsection{Training Loss for Consistency Policy}
Training consistency policy with consistency loss from the consistency model is nontrivial. Let us begin by assuming the presence of a pre-trained diffusion model with a score function, denoted as $s_{\phi^*}(a^k,k)$ with parameter $\phi^*$, representing the optimal diffusion-model-based policy. Following the consistency model, we can use a consistency policy to distill the inverse diffusion process. However, since we cannot access the data from the optimal policy, the assumptive pre-trained diffusion model is unavailable.

To tackle the above problem, we propose CPQL, which establishes a relationship between the above distillation process and the consistency training for reference policy with the Q-function based on Theorem \ref{th:final loss with ct}, with full proof in Appendix \ref{apx:train loss}. In order to determine the solution trajectory of action $\{a_t^k\}_{k\in[\epsilon, K]}$, we discretize the diffusion horizon $[\epsilon, K]$ into $M$ sub-intervals with boundaries $k_1 =\epsilon<k_2 <\cdots<k_M =K$. Note that here we use subscripts $m \in \{1, \cdots, M\}$ to denote time sub-intervals. To learn the consistency policy, we minimize the objective with stochastic gradient descent on the parameter $\theta$, while updating $\theta^{-}$ with exponential moving average.

\begin{theorem}
    \label{th:final loss with ct}
    Let $\Delta k=\max_{m\in [1, M-1]}\{|k_{m+1}-k_m|\}$. We have the assumptions: 1) Distance function $d$, value function $Q$ and $f_{\theta^{-}}$ are all twice continuously differentiable with bounded second derivatives; 2) There is a pre-trained score function representing the desired policy: $\forall k \in[\epsilon, K]: s_{\phi^*}(a^k,k)=\nabla\log p^k(a^k)$, which cannot be accessed; 3) $f_{\theta}$ satisfies the Lipschitz condition: there exists $L>0$ such that for all $k\in [\epsilon, K]$ and $x,y\in A$, we have $||f_{\theta}(x,k)-f_{\theta}(y,k)||_2\leq L||x-y||_2$. 
    The distillation loss for the distillation process of training the consistency policy is defined as:
    \begin{equation}
    \begin{aligned}
        L_{CD}(\theta,\theta^{-};\phi^*)= \mathbb{E}[d(f_{\theta}(a+k_{m+1}z,k_{m+1} | s),f_{\theta^{-}}(\hat{a}_{\phi^*}^{k_m},k_m|s))]
    \end{aligned}
    \end{equation}
    \noindent where $\hat{a}_{\phi^*}^{k_m}$ is calculated with Euler solver and the optimal score function $s_{\phi^*}(a^k,k)$. The training loss for training the consistency policy is sufficient to replace the distillation loss:
    \begin{equation}
    \begin{aligned}
    \label{eq:final loss}
        L(\theta,\theta^-)=\alpha L_{CT}(\theta,\theta^-)-\beta\mathbb{E}[ Q(s, \hat{a}^{\epsilon})]+o(\Delta k)
    \end{aligned}
    \end{equation}
    \noindent where the consistency loss for consistency training is defined as:
    \begin{equation}
    \begin{aligned}
        L_{CT}(\theta,\theta^{-}) = \mathbb{E}[d(f_{\theta}(a+k_{m+1}z,k_{m+1}|s),f_{\theta^{-}}(a+k_{m}z,k_{m}|s)]
    \end{aligned}
    \end{equation}
    \noindent and
    \begin{equation}
    \begin{aligned}
        \alpha=\mathbb{E}&[1+Q(s, a)\frac{k_{m+1}^2}{\lambda(a^{k_{m+1}}-a)^2}]\\
        \beta=\mathbb{E}&[d(f_{\theta}(a+k_{m+1}z,k_{m+1}|s),f_{\theta^-}(a+k_{m+1}z,k_{m+1}|s))\\
        &*\frac{k_{m+1}^2}{\lambda(a^{k_{m+1}}-a)^2})]\\
    \end{aligned}
    \label{eq:alpha beta}
    \end{equation}
\end{theorem}

The expectation is taken with respect to $(s,a)\sim \mathcal{D}_{offline}, \hat{a}^{\epsilon}\sim\pi_{\theta}, m\sim \mathcal{U}[1, M-1], z\sim \mathcal{N}(0,I)$, and $d$ stands for the Euclidean distance $d(x,y)=||x-y||^2_2$. 
Theorem \ref{th:final loss with ct} demonstrates that we can improve policy by distilling the assumed optimal score function using behavioral policy data and the Q-function.
Eq. (\ref{eq:final loss}) provides the loss function for training the consistency policy. $ L_{CT}(\theta,\theta^-)$ represents the consistent training loss regarding the behavior policy. Unlike consistency models, this part includes $ \alpha$ related to the Q-value as the weight to improve the probability of good actions sampled by the policy. The second term is the Q-value of the estimated action. Its purpose is to guide the policy to generate higher Q-value, leading to better performance. As seen from the above loss function, estimating the Q-value in the intermediate diffusion steps is no longer necessary, which avoids inaccurate guidance.

As experienced, we find that the loss function mentioned above may cause instability during the training process, resulting in the deterioration of policy performance. 
We speculate that the reason may be that consistency loss needs to sufficiently sample the diﬀusion time interval to achieve consistency with the original data. However, in practice, the limited sampling steps and the randomness of sampling may not meet this condition. In order to achieve more stable training, we propose a simpler loss function called reconstruction loss: 
\begin{equation}
    \begin{aligned}
    \label{eq:recon loss}
    L_{RC}(\theta,\theta^{-}) = \mathbb{E}[d(f_{\theta}(a+k_{m+1}z,k_{m+1}|s),a)]
    \end{aligned}
\end{equation}

This loss is intuitive. It drives the consistency policy to directly recover the original action, rather than indirectly achieving the recovery by maintaining consistency. We prove that this loss function has the same convergence objective with $L_{CT}(\theta,\theta^-)$ in Appendix \ref{apx:recon loss}. Therefore, we have the final optimization objective:
\begin{equation}
\begin{aligned}
    \label{eq:update}
    L(\theta,\theta^-)=&\alpha L_{RC}(\theta,\theta^-)\\
    &-\frac{\eta}{\mathbb{E}_{(s,a)\sim \mathcal{D}_{offline}}[Q(s,a)]}\mathbb{E}_{s\sim \mathcal{D}_{offline}, \hat{a}^{\epsilon}\sim\pi_{\theta}}[Q(s, \hat{a}^{\epsilon})]
\end{aligned}
\end{equation}

\begin{algorithm}[t]
    \caption{CPQL and CPIQL for Offline RL Tasks} 
    \begin{algorithmic}
        \STATE Initialize the policy network $\pi_{\theta}$, critic networks $V_{\psi}$, $Q_{\upsilon_1}$ and $Q_{\upsilon_2}$, and target networks $\pi_{\theta^{-}}$ , $Q_{\upsilon_1^{-}}$ and $Q_{\upsilon_2^{-}}$
        \FOR{each iteration}
            \STATE Sample transition batch $B=\{(s_t,a_t,r_t,s_{t+1})\}\sim \mathcal{D}_{offline}$
            \STATE \textcolor{gray!90}{\# Q-learning}
            \STATE Update $Q_{\upsilon_i}$ by Eq. (\ref{eq:q}) \textcolor{gray!90}{(CPQL)}
            \OR Update $V_{\psi}$, $Q_{\upsilon_i}$ by Eq. (\ref{eq:qv-v}) and Eq. (\ref{eq:qv-q}) \textcolor{gray!90}{(CPIQL)}\\
            \STATE \textcolor{gray!90}{\# Consistency policy learning}
            \STATE Sample $\hat{a}^{\epsilon}_t\sim\pi_{\theta}(a_t|s_t)$ by Eq. (\ref{eq:sample})
            \STATE Update $\pi_{\theta}$ by minimizing Eq. (\ref{eq:update})
            \STATE \textcolor{gray!90}{\# Update target networks}
            \STATE $\theta^{-}=\rho\theta^{-}+(1-\rho)\theta$
            \STATE $\upsilon_{i}^{-}=\rho\upsilon_{i}^{-}+(1-\rho)\upsilon_{i}$ for $i\in\{1,2\}$
        \ENDFOR
    \end{algorithmic} 
    \label{alg:offline}
\end{algorithm}
\noindent where $\frac{\eta}{\mathbb{E}_{(s,a)\sim \mathcal{D}_{offline}}[Q(s,a)]}$ corresponds to the $\beta$ in Eq. (\ref{eq:alpha beta}). In practice, we set $\alpha$ as an adjustable parameter ignoring the $Q$ value in $\alpha$, and $\eta$ is another adjustable parameter. The $(\alpha, \eta)$ are set according to the characteristics of different domains. 

\begin{table*}[ht]
  \centering
  \caption{The performance of CPQL and CPIQL and SOTA baselines on \textsf{D4RL} Locomotion and Adroit tasks. For Diffusion-QL, CPQL, and CPIQL, we here provide the "best" scores, representing the best performance during training. More experimental results, including the final performance and the standard deviation of each task, are provided in Appendix \ref{apx:more d4rl}. The bold values are the highest among each row.}
  \scalebox{0.95}{
  \begin{tabular}{lccccccccc|cccc}
    \toprule
    Dataset                      & TD3+BC & IQL  & SfBC  & IDQL  & Diffusion-QL   & EDP   & QGPO  & Diffuser & DD    & CPIQL          & CPQL           \\
    \bottomrule
    \toprule
    halfcheetah-medium-v2        & 48.3   & 47.4  & 45.9  & 49.7  & 51.1          & 52.1  & 54.1  & 42.8     & 49.1  & 55.3           & \textbf{57.9}  \\
    hopper-medium-v2             & 59.3   & 66.3  & 57.1  & 63.1  & 90.5          & 81.9  & 98.0  & 74.3     & 79.3  & {101.5}        & \textbf{102.1} \\
    walker2d-medium-v2           & 83.7   & 78.3  & 77.9  & 80.2  & 87.0          & 86.9  & 86.0  & 79.6     & 82.5  & {88.4}         & \textbf{90.5}  \\
    halfcheetah-medium-replay-v2 & 44.6   & 44.2  & 37.1  & 45.1  & 47.8          & 49.4  & 47.6  & 37.7     & 39.3  & \textbf{49.8}  & 48.1           \\
    hopper-medium-replay-v2      & 60.9   & 94.7  & 86.2  & 82.4  & 101.3         & 101.0 & 96.9  & 93.6     & 100.0 & \textbf{101.7} & \textbf{101.7} \\
    walker2d-medium-replay-v2    & 81.8   & 73.9  & 65.1  & 79.8  & \textbf{95.5} & 94.9  & 84.4  & 70.6     & 79.0  & 95.0           & 94.4           \\
    halfcheetah-medium-expert-v2 & 90.7   & 86.7  & 92.6  & 94.4  & 96.8          & 95.5  & 93.5  & 88.9     & 90.6  & 90.2           & \textbf{98.8}  \\
    hopper-medium-expert-v2      & 98.0   & 91.5  & 108.6 & 105.3 & 111.1         & 97.4  & 108.0 & 103.3    & 111.8 & 113.4          & \textbf{114.2} \\
    walker2d-medium-expert-v2    & 110.1  & 109.6 & 109.8 & 111.6 & 110.1         & 110.2 & 110.7 & 106.9    & 108.8 & \textbf{112.3} & 111.5          \\
    \midrule
    \rowcolor{gray!20} Average   & 75.3   & 77.0  & 75.6  & 79.1  & 87.9          & 85.5  & 86.5  & 77.5     & 81.8  & {89.7}         & \textbf{91.0}   \\
    \bottomrule
    \toprule
    pen-human-v1                 & 0.6    & 71.5  & -     & -     & 72.8          & 48.2  & -     & -        & -     & 58.2           & \textbf{89.3}  \\
    pen-cloned-v1                & -2.5   & 37.3  & -     & -     & 57.3          & 15.9  & -     & -        & -     & 77.4           & \textbf{83.3}  \\
    \midrule
    \rowcolor{gray!20} Average   & -1.0   & 54.4  & -     & -     & 65.1          & 32.1  & -     & -        & -     & 67.8           & \textbf{86.3}   \\
    \bottomrule
  \end{tabular}
  }
  \label{tab:offline scores}
\end{table*}

\subsection{Policy Evaluation}
For policy evaluation, the objective function uses the KL divergence to constrain the learned policy from accessing the OOD action. The Q-function of CPQL is learned in a conventional way, with the Bellman operator \cite{lillicrap2015continuous} and the double Q-learning trick \cite{hasselt2010double}. We build two Q-networks $Q_{\upsilon_1}$ and $Q_{\upsilon_2}$ and target networks $Q_{\upsilon_1^{-}}$ and $Q_{\upsilon_2^{-}}$ and minimizing the objective:
\begin{equation}
    \begin{aligned}
    L_{Q}(\upsilon_i)=&\mathbb{E}_{(s_t,a_t,s_{t+1})\sim \mathcal{D}_{offline}, a_{t+1}^{\epsilon}\sim \pi_{\theta}}[||r(s_t,a_t)\\
    &+\gamma \min_{j=1,2}Q_{\upsilon_j^{-}}(s_{t+1},a_{t+1}^{\epsilon})-Q_{\upsilon_i}(s_t,a_t)||^2]
    \end{aligned}
\label{eq:q}
\end{equation}

To evaluate the effectiveness of the consistency policy in constraining the sampling of OOD actions that cause overestimation of the Q-value, we further propose Consistency Policy Implicit Q-Learning (CPIQL) as a comparison. Using implicit Q-learning \cite{kostrikov2021offline} can better restrict accessing the OOD action. Meanwhile, implicit Q-learning avoids sampling the action at the next timestep when calculating the critic loss function, reducing the time needed for critic training. CPIQL builds the value-network $V_{\psi}$ in addition. Typically, the value function $Q$ and $V$ are given by:
\begin{equation}
    L_{V}(\psi)=\mathbb{E}_{(s_t,a_t)\sim \mathcal{D}_{offline}}[L_2^\tau(\min_{j=1,2}Q_{\upsilon_j^{-}}(s_t,a_t)-V_{\psi}(s_t))]
\label{eq:qv-v}
\end{equation}
\begin{equation}
    \begin{aligned}
    L_{Q}(\upsilon_i)=\mathbb{E}_{(s_t,a_t,s_{t+1})\sim \mathcal{D}_{offline}}[||r(s,a)+\gamma V_{\psi}(s_{t+1})-&Q_{\upsilon_i}(s_t,a_t)||_2^2], \\
    &for\, i=1,2
    \label{eq:qv-q}
    \end{aligned}
\end{equation}
where $L_2^\tau(u)=|\tau-\mathds{1}(u<0)|u^2$ is expectile regression function.

In contrast to other RL methods based on the diffusion model, the consistency policy affords the distinct advantage of facilitating one-step sampling. This feature significantly reduces the time costs associated with both the sampling process and the computation of Q-function gradient back-propagation, resulting in better time efficiency. We summarize the complete algorithm procedure of CPQL and CPIQL in Algorithm \ref{alg:offline} for offline tasks.

\subsection{Extension for Online RL}
Previous sessions have demonstrated how CPQL learns the policy of Eq. (2) to solve offline RL problems. In this session, we will demonstrate that the CPQL method can seamlessly extend to online RL problems. We define the objective for online tasks as:
\begin{equation}
\begin{aligned}
    \mathcal{J}(\pi)=\mathbb{E}_{s_t\sim \mathcal{D}_{r}}[\mathbb{E}_{a_t\sim\pi(\cdot|s_t)}[Q_{\upsilon}(s_t,a_t)] \\
    -\lambda D_{KL}(\pi(\cdot|s_t)||\pi_r(\cdot|s_t))]
\label{eq:online obj}
\end{aligned}
\end{equation}
\noindent where $\mathcal{D}_{r}$ refers to the replay buffer, and $\pi_r$ is the policy for collecting the data in $\mathcal{D}_{r}$. 
When $\pi_r$ is the uniform policy, the above problem is equivalent to maximum entropy RL \cite{haarnoja2018soft}.
It is readily observable that the objective function for online RL tasks bears a striking resemblance to Eq. (\ref{eq:batch constrained RL}), and Eq. (\ref{eq:optimal policy}) can also serve as a closed-form solution for Eq. (\ref{eq:online obj}) with $\mathcal{D}_{r}$ replacing $\mathcal{D}_{offline}$ and $\pi_r$  replacing $\mu$. Therefore, CPQL can be seamlessly extended to online tasks. In maximum entropy RL methods, careful tuning of $\lambda$ is necessary to strike a balance between exploration and exploitation. 
% Our approach deviates from this setup, with the reference policy being set to the behavioral policy collecting experience in the replay buffer. 
In CPQL, this situation is alleviated. 
Firstly, our consistency policy is inherently stochastic, and we find in experiments that this inherent randomness allows for sufficient exploration of the environment without the need for additional exploration strategies. Secondly, the data collected by this policy incorporates this stochasticity and makes the consistency policy maintain this randomness through Eq. (\ref{eq:update}). Lastly, as the policy asymptotically converges, the proportion of good samples in the data increases, reducing the randomness in the replay buffer, consequently decreasing the policy's randomness and achieving policy convergence. It is worth noting that the aforementioned process does not require manual adjustment of $\lambda$; rather, it is implicit in the CPQL policy iteration. Furthermore, unlike previous maximum entropy reinforcement learning methods, we find that CPQL's performance is not sensitive to $\lambda$ for different tasks, greatly alleviating the tuning complexity associated with maximum entropy reinforcement learning. We summarize the complete algorithm procedure of CPQL in Algorithm \ref{alg:online} for online tasks, and we also provide an illustrative description of how CPQL works for both offline and online RL tasks in Appendix \ref{apx: policy}.

\begin{algorithm}[t]
    \caption{CPQL for Online RL Tasks} 
    \begin{algorithmic}
        \STATE Initialize the policy network $\pi_{\theta}$ , critic networks $Q_{\upsilon_1}$ and $Q_{\upsilon_2}$ , and target networks $\pi_{\theta^{-}}$ , $Q_{\upsilon_1^{-}}$ and $Q_{\upsilon_2^{-}}$
        \STATE Initialize the dataset $\mathcal{D}_r\leftarrow\emptyset$
        \STATE \textcolor{gray!90}{\# Warm up}
        \FOR{$i \in 0, \cdots, W$}
            \STATE Generate $\hat{a}^{\epsilon}_t$ by Eq. (\ref{eq:sample}) with $s_t$
            \STATE Play $\hat{a}^{\epsilon}_t$ and get $s_{t+1}\sim\mathds{P}(\cdot|s_t,\hat{a}^{\epsilon}_t)$
            \STATE $\mathcal{D}_r\leftarrow\mathcal{D}_r\cup\{s_t, \hat{a}^{\epsilon}_t, r_t, s_{t+1}\}$
        \ENDFOR
        \FOR{each iteration}
            \STATE \textcolor{gray!90}{\# Update the dataset}
            \STATE Sample $\hat{a}^{\epsilon}_t\sim\pi_{\theta}(a_t|s_t)$ by Eq. (\ref{eq:sample})  with $s_t$
            \STATE Play $\hat{a}^{\epsilon}_t$ and get $s_{t+1}\sim\mathds{P}(\cdot|s_t,\hat{a}^{\epsilon}_t)$
            \STATE $\mathcal{D}_r\leftarrow\mathcal{D}_r\cup\{s_t, \hat{a}^{\epsilon}_t, r_t, s_{t+1}\}$
            \STATE \textcolor{gray!90}{\# Policy training}
            \STATE Sample transition batch $B=\{(s_{t},a_{t},r_{t},s_{t+1})\}\sim \mathcal{D}_r$
            \STATE \textcolor{gray!90}{\# Q-learning}
            \STATE Update $Q_{\upsilon_i}$ by Eq. (\ref{eq:q}) % with $ \mathcal{D}_r$ % replacing $\mathcal{D}_{offline}$
            \STATE \textcolor{gray!90}{\# Consistency policy learning}
            \STATE Update policy by minimizing Eq. (\ref{eq:update}) % with $ \mathcal{D}_r$ % replacing $\mathcal{D}_{offline}$
            \STATE \textcolor{gray!90}{\# Update target networks}
            \STATE $\theta^{-}=\rho\theta^{-}+(1-\rho)\theta$
            \STATE$\upsilon_{i}^{-}=\rho\upsilon_{i}^{-}+(1-\rho)\upsilon_{i}$ for $i\in\{1,2\}$
        \ENDFOR
    \end{algorithmic} 
    \label{alg:online}
\end{algorithm}

%%%%%%%%%%%%%%%%%%%%%%%%%%%%%%%%%%%%%%%%%%%%%%%%%%%%%%%%%%%%%%%%%%%%%%%%

\begin{table*}[ht]
  \centering
  \caption{The performance of CPQL and SOTA baselines on \textsf{dm\_control} tasks under 500K environment steps. The results of MPO, DMPO, D4PG, and DreamerV3 are from the paper of DreamerV3 \cite{hafner2023mastering}. The bold values are the highest among each row.}
  \vspace{-1em}
  \scalebox{0.97}{
  \begin{tabular}{lccccccc|c}             \\
    \toprule
    Tasks                      & TD3              & SAC     & PPO             & MPO   & DMPO  & D4PG           & DreamerV3 & CPQL \\
    \bottomrule
    \toprule
    Acrobot Swingup            & 46.8             & 33.2    & 34.4            & 80.6  & 98.5  & 125.5          & 154.5     & \textbf{183.1} \\
    Cartpole Balance           & 982.6            & 961.9   & 997.6           & 958.4 & 998.5 & 998.8          & 990.5     & \textbf{999.2} \\
    Cartpole Balance Sparse    & 992.0            & 993.5   & \textbf{1000.0} & 998.0 & 994.0 & 979.6          & 996.8     & \textbf{1000.0}\\
    Cartpole Swingup           & 829.2            & 781.4   & 760.7           & 857.7 & 857.8 & \textbf{874.6} & 850.0     & 860.7 \\
    Cheetah Run                & 546.2            & 530.9   & 560.4           & 612.3 & 581.6 & 623.5          & 575.9     & \textbf{727.6} \\
    Finger Spin                & 847.1            & 825.4   & 369.9           & 766.9 & 744.3 & 818.4          & 937.2     & \textbf{965.4} \\
    Finger Turn Easy           & 337.6            & 371.4   & 275.2           & 430.4 & 593.8 & 524.5          & 745.4     & \textbf{874.1} \\
    Finger Turn Hard           & 334.4            & 344.8   & 5.06            & 250.8 & 384.5 & 379.2          & 841.0     & \textbf{864.6} \\
    Hopper Hop                 & 40.0             & 41.7    & 0.0             & 37.5  & 71.5  & 67.5           & 111.0     & \textbf{130.1} \\
    Hopper Stand               & 322.7            & 270.9   & 2.2             & 279.3 & 519.5 & 755.4          & 573.2     & \textbf{902.1} \\
    Reacher Easy               & 968.8            & 973.4   & 541.6           & 954.4 & 965.1 & 941.5          & 947.1     & \textbf{981.2} \\
    Reacher Hard               & \textbf{965.7}   & 928.2   & 518.0           & 914.1 & 956.8 & 932.0          & 936.2     & 963.6 \\
    Walker Run                 & 274.5            & 445.9   & 131.7           & 539.5 & 462.9 & 593.1          & 632.7     & \textbf{683.8} \\
    Walker Stand               & 957.7            & 973.4   & 528.4           & 960.4 & 971.6 & 935.2          & 956.9     & \textbf{983.6} \\
    Walker Walk                & 934.7            & 922.8   & 408.9           & 924.9 & 933.1 & \textbf{965.1} & 935.7     & 952.6 \\
    \midrule
    \rowcolor{gray!20} Average & 625.3            & 626.6   & 408.9           & 637.7 & 675.6 & 700.9          & 745.6     & \textbf{804.9} \\
    \bottomrule
  \end{tabular}
  }
  \label{tab:online scores dm control}
\end{table*}

\section{Experiments}
In this section, we conduct several experiments on the \textsf{D4RL} benchmark \cite{fu2020d4rl}, \textsf{dm\_control} tasks \cite{tunyasuvunakool2020}, \textsf{Gym MuJoCo} tasks \cite{todorov2012physics} to evaluate the performance and time efficiency of the consistency policy. We also provide various ablation studies for a better understanding of how the hyperparameter $(\alpha, \eta)$ in Eq. (\ref{eq:update}) and new loss Eq. (\ref{eq:recon loss}) affect the performance. Throughout this paper, the results are reported by averaging five random seeds. For a detailed look at the experimental setup and corresponding hyperparameters, please refer to Appendix \ref{apx: exp setup}.

\paragraph{Baselines for Offline RL Tasks}
For offline scenario, we evaluate on two domains of \textsf{D4RL} benchmark and compare CPQL with current methods that achieve SOTA performance, including Q-learning with policy constraints such as TD3+BC \cite{fujimoto2021minimalist}, implicit Q-learning such as IQL \cite{kostrikov2021offline}, and diffusion-model-based methods, such as Diffuser \cite{janner2022planning}, DecisionDiffuser (DD) \cite{ajay2022conditional}, SfBC \cite{chen2022offline}, IDQL \cite{hansen2023idql}, Diffusion-QL\cite{wang2022diffusion}, EDP \cite{kang2023efficient}, QGDO \cite{lu2023contrastive}. Notably, Diffuser and DD employ the diffusion model to capture trajectory dynamics, while the remaining methods focus on modeling policy distributions.

\paragraph{Baselines for Online RL Tasks}
For the online scenario, we evaluate our method on several tasks, such as \textsf{dm\_control} tasks and \textsf{Gym MuJoCo} tasks.
We compared CPQL with the current methods of achieving SOTA. These methods include off-policy methods such as TD3 \cite{fujimoto2018addressing}, SAC \cite{haarnoja2018soft}, MPO \cite{abdolmaleki2018maximum}, D4PG \cite{barth2018distributed}, DMPO \cite{abdolmaleki2020distributional}, and on-policy methods such as PPO \cite{schulman2017proximal}. In addition, the comparison methods also include model-based methods such as DreamerV3 \cite{hafner2023mastering}, and diffusion-model-based methods such as DIPO \cite{yang2023policy}.

\subsection{Overall Results}
In this section, we illustrate the overall results on the offline tasks (\textsf{D4RL}) and the online tasks (\textsf{dm\_control}, \textsf{Gym Mujoco}), showing that CPQL achieves competitive performance on both offline and online tasks compared with current SOTA methods. All training curves for CPQL and CPIQL can be found in Appendix \ref{apx:train curve}.

\subsubsection{Results on \textsf{D4RL} (Offline Tasks)}
Firstly, we evaluate the performance of the methods on offline \textsf{D4RL} tasks. The experimental results are shown in Table \ref{tab:offline scores}. From the experimental results, we can see that the methods based on the diffusion model have significant advantages in performance compared with the unimodal distribution policy method (TD3+BC and IQL). Compared with Diffusion-QL and EDP, QGPO achieves excellent performance by using an accurate guided value function with contrastive loss. Our proposed CPQL avoids inaccurate guidance problems by modeling ODE trajectory mapping, thus achieving competitive results by approximately 4\% improvement (compared with Diffusion-QL). By comparing CPQL and CPIQL, we can find that consistency policy can effectively constrain the sampling of OOD actions without the help of implicit Q-learning. Additionally, CPIQL outperforms IQL by a significant 16.5\% on locomotion tasks, indicating that the strong representation ability of consistency policy can effectively achieve policy improvement with accurate diffusion guidance. In conclusion, CPQL achieves better guidance with value function while meeting constraints, thus making their performance outstanding on offline tasks.

\begin{figure*}[htbp]
\centering
\includegraphics[width=14cm]{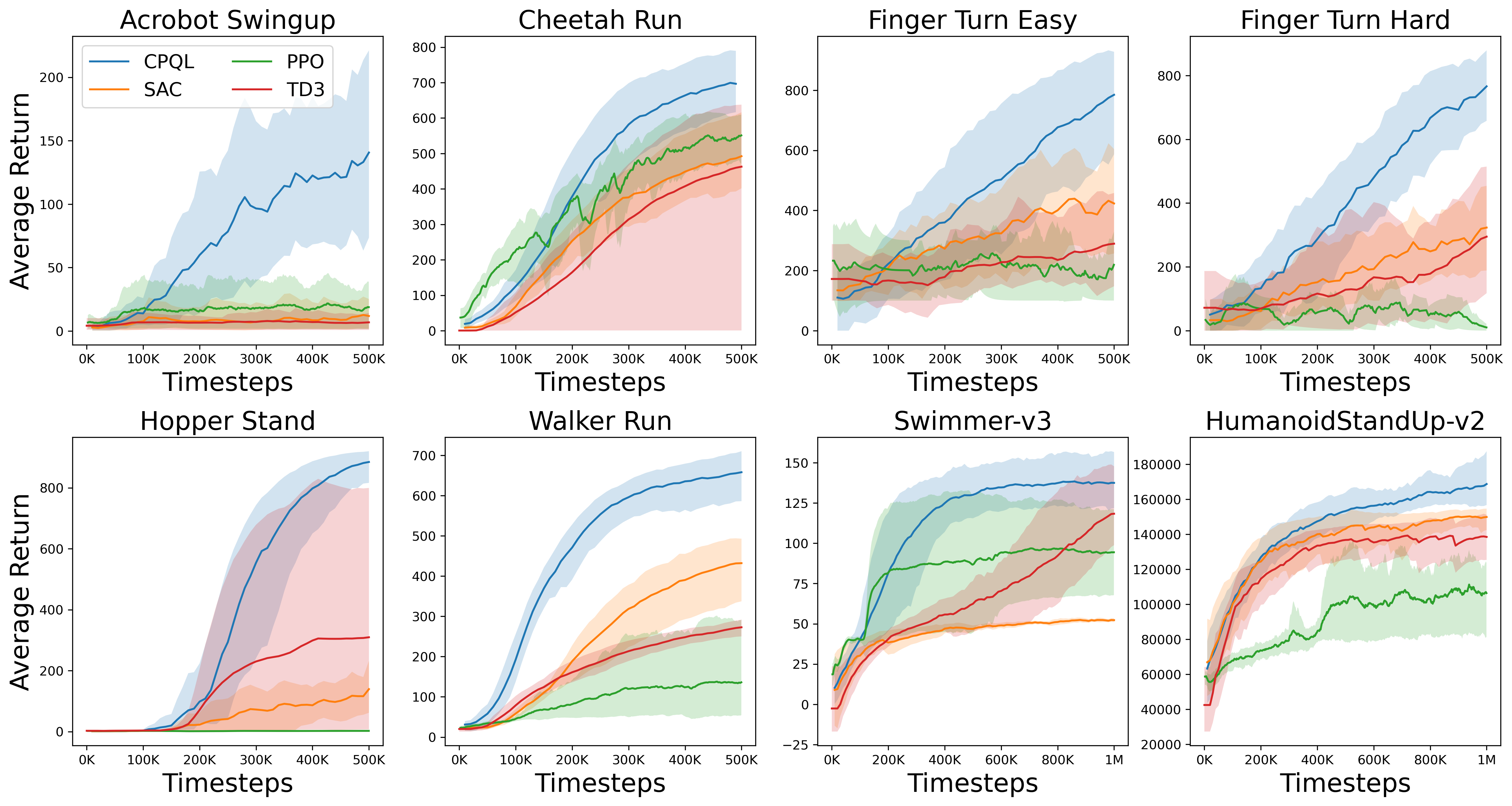}
\caption{Training curves for 8 online tasks, including 6 tasks from \textsf{dm\_control} (top row, and leftmost two of the bottom row) and 2 tasks from \textsf{Gym Mujoco} (rightmost two of the bottom row). CPQL, PPO, SAC and TD3 are compared on each task with 5 random seeds. }
\label{fig:online result}
\end{figure*}

\subsubsection{Results on \textsf{dm\_control} (Online Tasks)}
Next, we evaluate the methods on 15 \textsf{dm\_control} tasks. Part of the training curves and all experimental results are shown in Figure \ref{fig:online result} and Table \ref{tab:online scores dm control}, respectively. As shown in Figure \ref{fig:online result}, CPQL has shown an improvement in both training stability and sample efficiency with better performance when compared to other algorithms under a small number of interactions with the environment (500 K). Also, the experimental results show that CPQL performance outperforms previous SOTA methods, including DreamerV3 on 15 tasks, especially on "Cheetah Run" and "Finger Turn Easy" tasks. On average, there has been an improvement of approximately 8\% as compared to DreamerV3. In addition, CPQL directly models the policy and is not plagued by inaccurate diffusion guidance, resulting in better performance on complex tasks, with great representation ability for the policy. It is worth mentioning that DMPO and D4PG have adopted the distributional value function \cite{BellemareDM17}, which has proved to achieve better performance than the expected value function. While CPQL only uses the expected value function, it has shown significant performance leadership, which also means the importance of exploration in RL, and the consistency model has great potential in RL.

\begin{table}[t]
  \centering
  \caption{The performance of CPQL and SOTA baselines on Gym MuJoCo tasks under 1M environment steps. The bold values are the highest among each row.}
  \vspace{-1em}
  \scalebox{0.77}{
  \begin{tabular}{lcccc|c}             \\
    \toprule
    Tasks                      & TD3      & SAC      & PPO      & DIPO     & CPQL \\
    \bottomrule
    \toprule
    Swimmer-v3                 & 108.3    & 51.7     & 94.6     & 72.2     & \textbf{137.2}    \\
    Walker2d-v3                & 4127.0   & 4631.9   & 3751.8   & 4409.6   & \textbf{5139.9}   \\
    Ant-v3                     & 5421.9   & 5665.5   & 2921.0   & 5620.2   & \textbf{6209.4}   \\
    HalfCheetah-v3             & 10779.9  & 11287.9  & 2449.5   & 10475.2  & \textbf{12195.2}  \\
    Humanoid-v3                & 5253.0   & 4993.3   & 704.1    & 4878.5   & \textbf{5394.6}   \\
    HumanoidStandup-v2         & 136897.9 & 150934.8 & 105654.4 & 145350.2 & \textbf{174480.6} \\
    \bottomrule
  \end{tabular}
  }
  \label{tab:online scores mujoco}
%   \vspace{-1.8em}
\end{table}

\subsubsection{Results on \textsf{Gym MuJoCo} (Online Tasks)}
We also evaluate the methods on 6 tasks of \textsf{Gym MuJoCo}. Part of the training curves and all experimental results are shown in Figure \ref{fig:online result} and Table \ref{tab:online scores mujoco}. Consistent with the conclusion on the \textsf{dm\_control} tasks, the methods based on the diffusion model maintain dominance on multiple RL tasks. As we can see from the rightmost two graphs of the bottom row in Figure \ref{fig:online result}, CPQL has significant performance advantages compared with other algorithms within 1M interactions. It is worth mentioning that DIPO proposed the concept of action gradient and used a diffusion model to fit the replay buffer updated by action gradient. However, on one hand, it will be affected by the coverage of the initial dataset; on the other hand, only fitting the optimal data is likely to cause the diffusion model to lose sampling diversity to reduce its exploration ability. Therefore, CPQL performance is better than DIPO.

\begin{figure}[htbp]
\centering
\includegraphics[width=8cm]{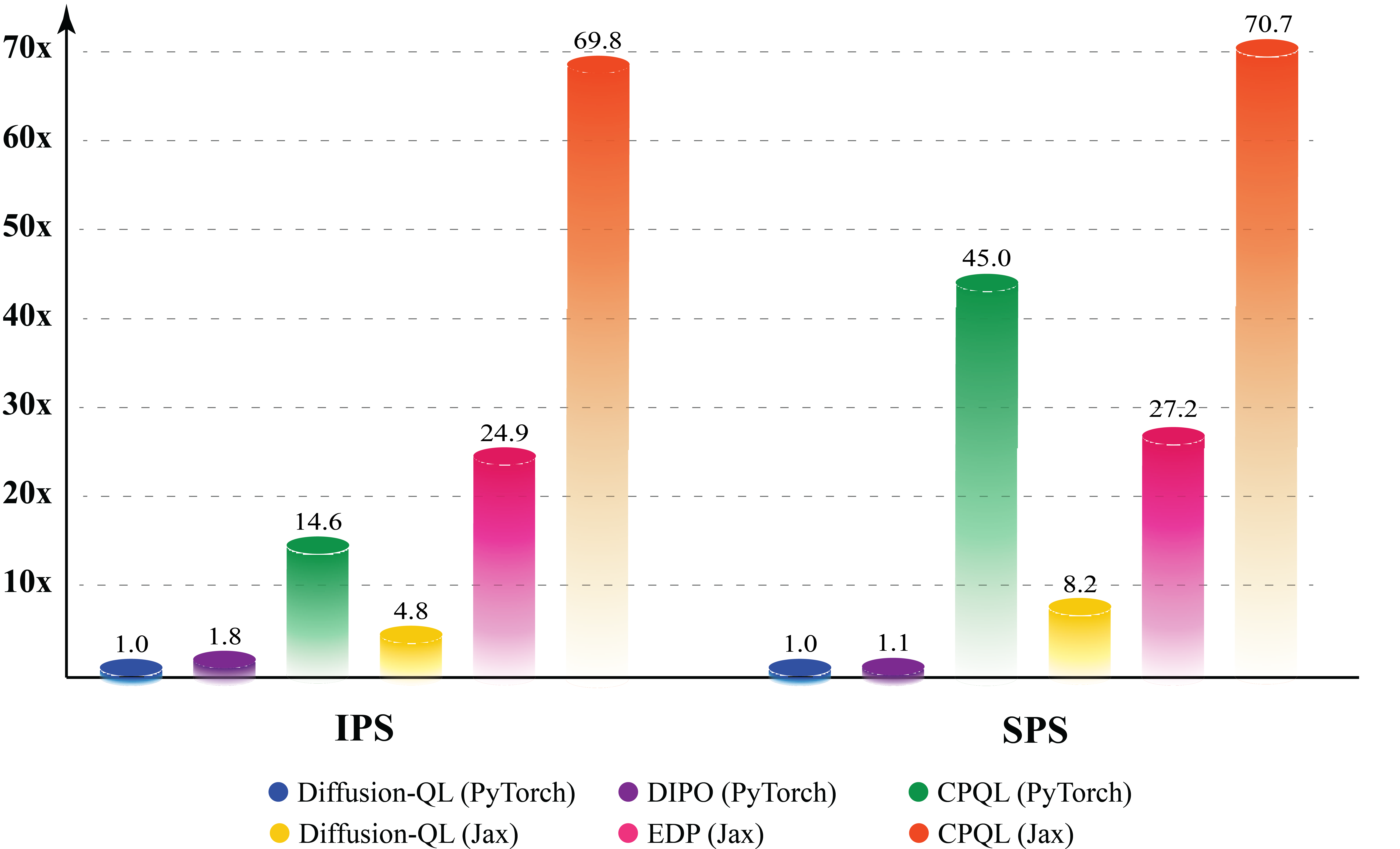}
\caption{Training and inference speedup on \textsf{D4RL} locomotion tasks. We choose the Diffusion-QL (Pytorch) as the baseline with all original data provided in Appendix \ref{apx:more d4rl}. Implementation for Diffusion-QL(Pytorch), DIPO(Pytorch), EDP(Jax) are from the official repository, with websites listed in Appendix \ref{apx: exp setup}. And we implement Diffusion-QL(Jax) by ourselves for comparison.}
\label{fig:ips sps}
% \vspace{-1.5em}
\end{figure}

\subsection{Time Efficiency}
We then evaluate the time efficiency of policy training and inference. Following the evaluation criterion in EDP \cite{kang2023efficient}, we compare methods including Diffusion-QL, EDP, DIPO, and CPQL by iterations-per-second (IPS) for training speed and steps-per-second (SPS) for inference speed. We choose "walker2d-medium-expert-v2" as the testbed and run each method for 100K iterations of policy updates to calculate the IPS. Then, we sample 100K transitions by interacting with the environment to calculate the SPS.

From the result in Figure \ref{fig:ips sps}, CPQL's training and sampling speeds are far faster than other methods based on the diffusion model. Especially compared to Diffusion-QL, CPQL has improved its training speed by nearly 15 times and inference speed by nearly 45 times. Even though EDP also uses one-step diffusion in actor training to approximate the multi-step diffusion process in Diffusion-QL, as well as DPM-solver to speed up sampling, multi-step sampling is still necessary during Q-function training and inference sampling. Hence the consistency policy also has a noteworthy improvement in both the training and inference speed compared to the EDP.

\begin{figure}[t]
\centering
    \subfloat[$\eta=1$ and different $\alpha$ for selected online RL tasks]
    {
        \label{subfig:ablation alpha}
        \includegraphics[width=8cm]{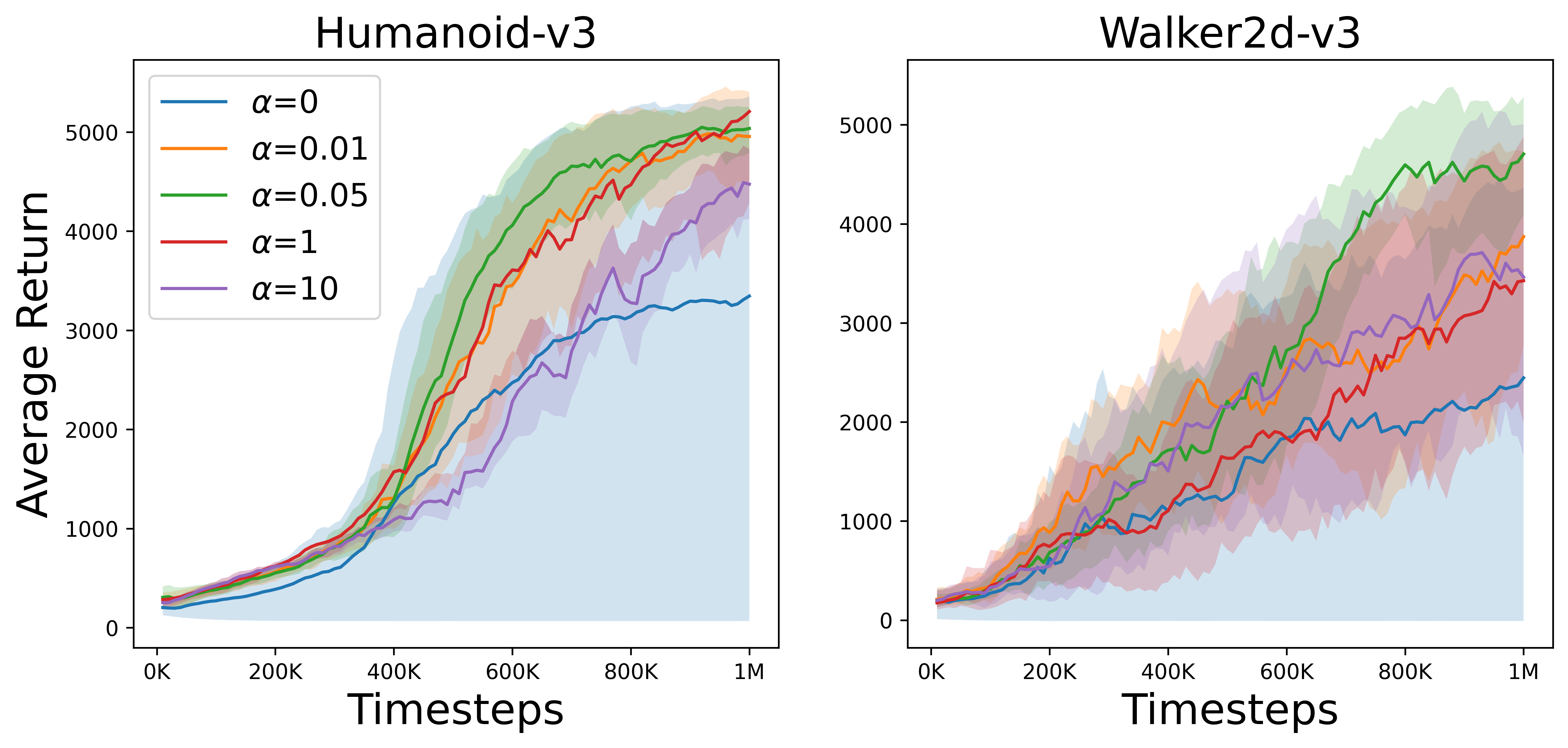}
    }
    \hfill
    \subfloat[$\alpha=1$ and different $\eta$ for selected offline RL tasks] 
    {
        \label{subfig:ablation eta}
        \includegraphics[width=8cm]{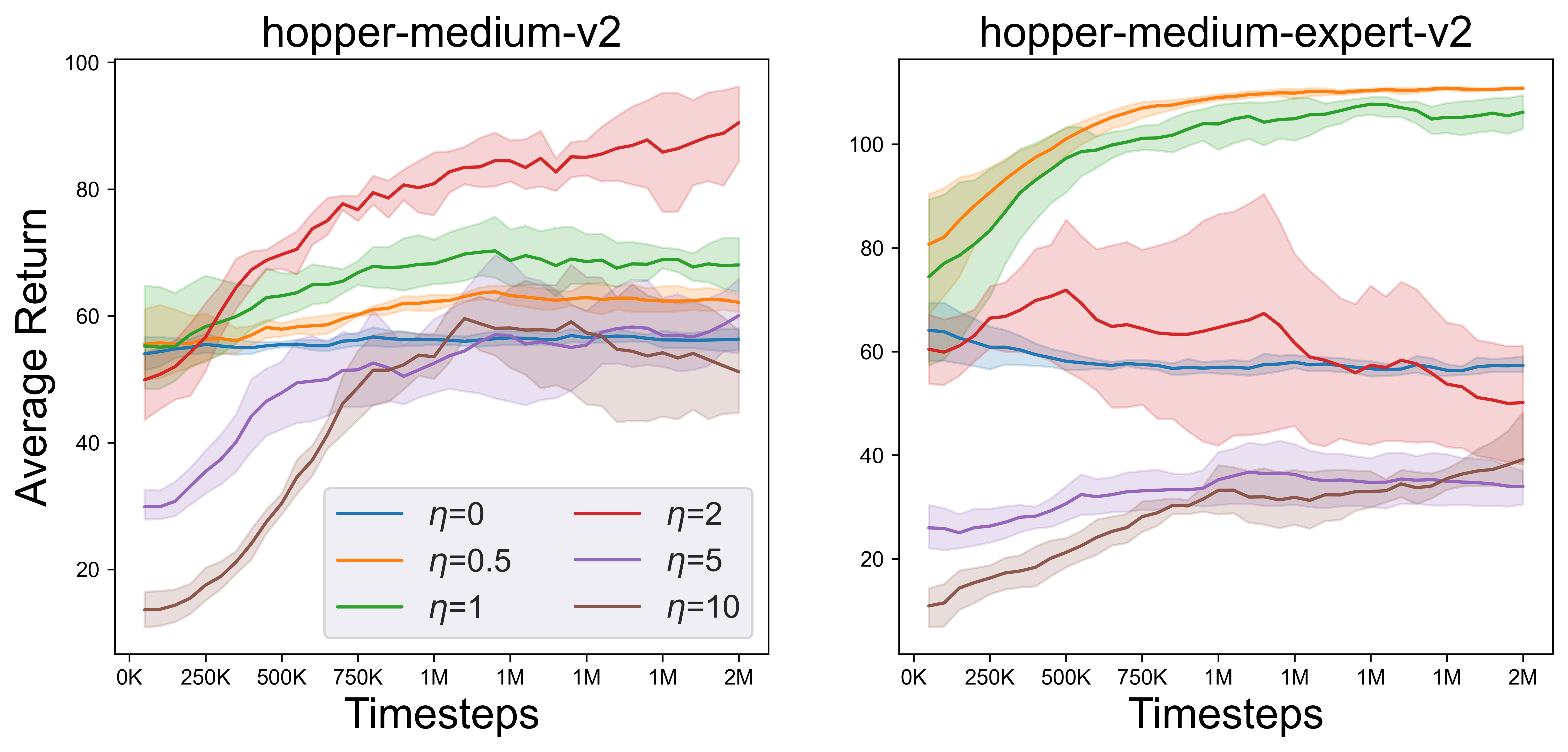}
    }
\caption{The effect of the different hyperparameters $\eta$ and $\alpha$ on the online and offline policy learning.}
\label{fig:ablation eta alpha}
\end{figure}

\subsection{Ablation Study}
In this section, we conduct various ablation studies on both online and offline continuous control tasks to determine the factors affecting the performance, including evaluating the impact of different parameters on the policy training and analyzing the stability introduced by reconstruction loss during consistency policy training.

\paragraph{Hyperparameter}
$\eta$ and $\alpha$ respectively control the impact of the Q-function on policy and the KL divergence between learned policy and reference policy. In online tasks, we set $\eta=1$ and investigate the impact of different $\alpha$ on policy performance. The experimental results are shown in Figure \ref{fig:ablation eta alpha} (a). According to the results of two different tasks, an appropriate $\alpha$ is crucial for obtaining better performance. If $\alpha$ is too small or too large, it will cause performance degradation. But the reasons for performance degradation are different. When $\alpha$ is small, the consistency policy is more affected by the gradient of the Q-function, leading to premature loss of sampling diversity and affecting the exploration ability. When $\alpha$ is relatively large, the data distribution greatly affects the consistency policy, which leads to the policy ignoring the guidance of the Q-function.

In offline tasks, we set $\alpha=1$ and investigate the impact of different $\eta$ on policy performance. The results are shown in Figure \ref{fig:ablation eta alpha} (b). The results of two different tasks indicate that when the $\eta$ ratio is small, the impact of the Q-function decreases, and the performance of the policy tends to approach that of behavioral cloning. When $\eta$ is relatively large, the constraint effect of the dataset on the policy decreases, leading to more OOD actions in the sampling process, resulting in deviation of Q-value estimation, which also leads to performance degradation. 

\begin{figure}[t]
\centering
\includegraphics[width=8cm]{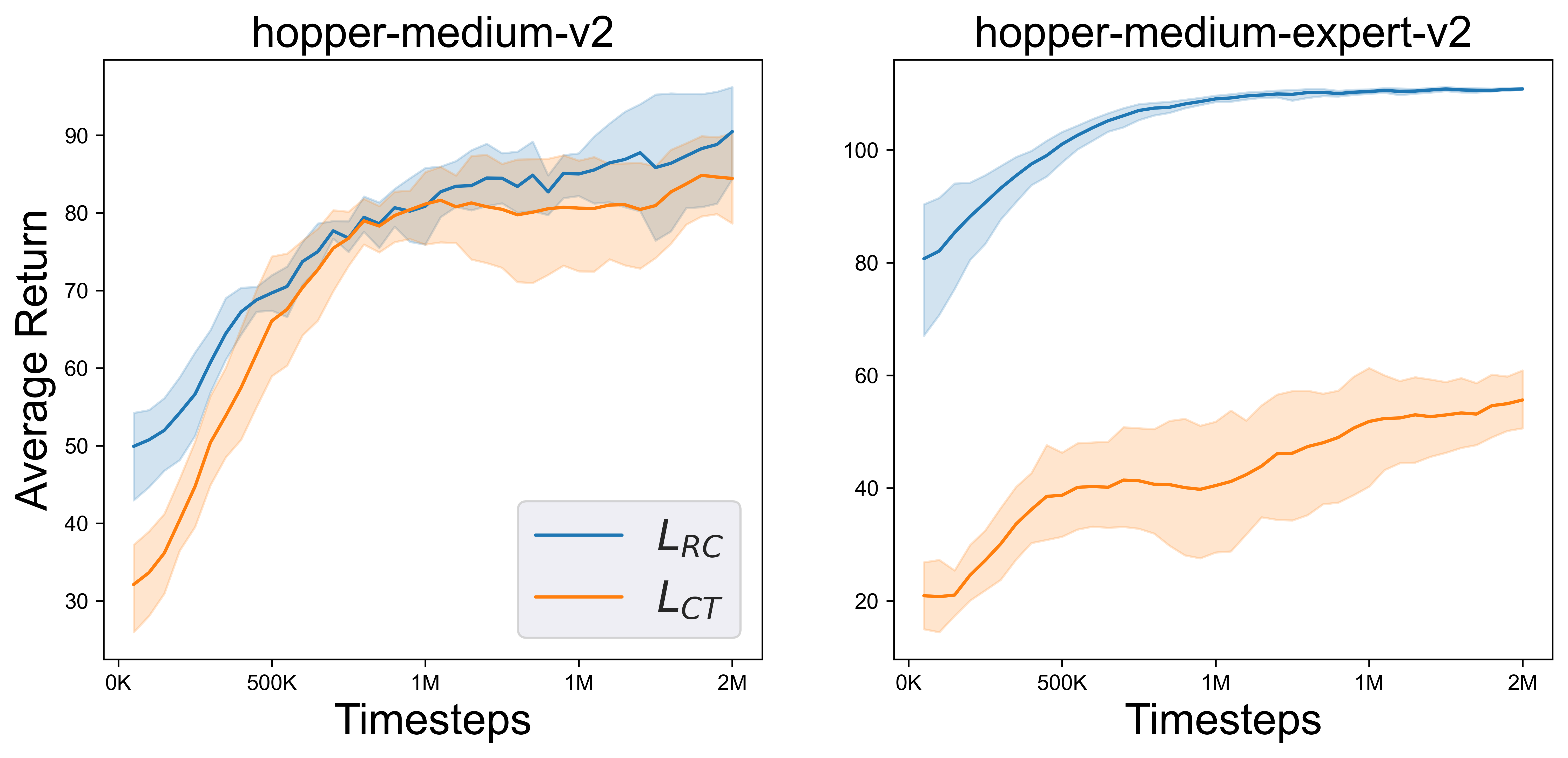}
\caption{The effect of the different training losses on policy learning.}
\vspace{-1.0em}
\label{fig:ablation loss}
\end{figure}

\paragraph{Training Loss for Consistency Model}
We propose replacing consistency training loss $L_{CT}$ with reconstruction loss $L_{RC}$. Here we study the influence of two different loss functions on policy training as shown in Figure \ref{fig:ablation loss}. During the policy training, we find that consistency training loss would lead to policy collapse. The results show that reconstruction loss makes the training process more stable and can achieve better results.

%%%%%%%%%%%%%%%%%%%%%%%%%%%%%%%%%%%%%%%%%%%%%%%%%%%%%%%%%%%%%%%%%%%%%%%%

\section{Conclusion}
In this work, we propose the time-efficiency consistency policy with Q-learning (CPQL), which constructs the mapping from the PF ODE trajectories to the desired policy and achieves policy improvement with accurate guidance for both offline and online RL tasks. We also introduce an empirical loss for stabilizing consistency policy training. Experimental results show that CPQL achieves about 4\% improvement on \textsf{D4RL} locomotion tasks (compared to Diffusion-QL) and 8\% improvement on \textsf{dm\_control} tasks (compared to DreamerV3). Meanwhile, CPQL significantly improves inference speed by nearly 45 times compared to Diffusion-QL. We show the potential of the consistency model in RL and believe this impact is profound. Its efficient sampling speed has greatly expanded real-time applications such as robot control based on the diffusion model. Of course, there are still many problems, especially for online RL. For instance, how to better control the diversity of diffusion model in the training process and whether it is possible to build a multi-step consistency policy or not to enhance the representation of the policy is worth studying.

%%%%%%%%%%%%%%%%%%%%%%%%%%%%%%%%%%%%%%%%%%%%%%%%%%%%%%%%%%%%%%%%%%%%%%%%

%%% The acknowledgments section is defined using the "acks" environment
%%% (rather than an unnumbered section). The use of this environment 
%%% ensures the proper identification of the section in the article 
%%% metadata as well as the consistent spelling of the heading.

\begin{acks}
This work is supported by the National Natural Science Foundation
of China (NSFC) under Grants No. 62136008, No. 62103409, the Strategic Priority Research Program of Chinese Academy of Sciences (CAS) under Grant
XDA27030400 and in part by the International Partnership Program of the
Chinese Academy of Sciences under Grant 104GJHZ2022013GC.
\end{acks}

%%%%%%%%%%%%%%%%%%%%%%%%%%%%%%%%%%%%%%%%%%%%%%%%%%%%%%%%%%%%%%%%%%%%%%%%

%%% The next two lines define, first, the bibliography style to be 
%%% applied, and, second, the bibliography file to be used.

% \bibliographystyle{ACM-Reference-Format} 
% \bibliography{references}


\bibliographystyle{ACM-Reference-Format}
% \bibliography{reference}

%%%%%%%%%%%%%%%%%%%%%%%%%%%%%%%%%%%%%%%%%%%%%%%%%%%%%%%%%%%%%%%%%%%%%%%%

\newpage
\onecolumn

\section*{Appendices}
\appendix

\setcounter{equation}{18}
\setcounter{table}{3}
\setcounter{figure}{5}

\section{Mathematical Proofs}
\paragraph{Notations}
For a decision-making problem in RL, we use $s$ and $a$ to represent the state and action, respectively. $p(s'|s,a)$ represents the transition probability from state $ s$ to next state $s'$ after taking the action $a$. We use subscripts $t\in \{1,\cdots,T\}$ to denote trajectory timesteps and we use superscripts $k\in [0, K]$ to denote diffusion timesteps with $m \in \{1, \cdots, M\}$ denoting time sub-intervals. 

Given a multi-variate function $h(x,y)$, we let $\partial_1(x,y)$ denote the Jacobian of $h$ over $x$, and analogously $\partial_2(x,y)$ denote the Jacobian of $h$ over $y$. Unless otherwise stated, $a$ is supposed to be a random variable sampled from the data distribution $p_{data}(a)$, $m$ is sampled uniformly at random from $[1, \cdots, M-1]$.

\paragraph{Preliminaries}
We use $f_{\theta}(a^k,k|s)$ to denote a consistency policy parameterized by $\theta$. The outputs of the consistency policy are consistent for all arbitrary pairs of $(a^k,k)$ on the same PF ODE with an empirical estimate: $\frac{da^k}{dk}=-ks_{\phi}(a^k,k)$ where the pre-trained score function is parameterized by $\phi$.  By utilizing a numerical ODE solver such as an Euler solver, we can accurately estimate $a^{k_m}$ from $a^{k_{m+1}}$, which is denoted as:
\begin{equation}
\begin{aligned}
    \hat{a}_{\phi^*}^{k_m}=a^{k_{m+1}}+(k_m-k_{m+1})s_{\phi^*}(a^{k_{m+1}},k_{m+1})
\end{aligned}
\end{equation}

The distillation loss is defined as:
\begin{equation}
\begin{aligned}
    L_{CD}(\theta,\theta^{-};\phi^*)= \mathbb{E}[d(f_{\theta}(a+k_{m+1}z,k_{m+1} | s),f_{\theta^{-}}(\hat{a}_{\phi^*}^{k_m},k_m|s))]
\end{aligned}
\end{equation}

The consistency loss is defined as:
\begin{equation}
\begin{aligned}
    L_{CT}(\theta,\theta^{-}) = \mathbb{E}[d(f_{\theta}(a+k_{m+1}z,k_{m+1}|s),f_{\theta^{-}}(a+k_{m}z,k_{m}|s)]
\end{aligned}
\end{equation}

\subsection{The Proof of Training Loss for Consistency Policy}
\label{apx:train loss}
\begin{theorem}
    Let $\Delta k=\max_{m\in [1, M-1]}\{|k_{m+1}-k_m|\}$. We have the assumptions: 1) Distance function $d$, value function $Q$ and $f_{\theta^{-}}$ are all twice continuously differentiable with bounded second derivatives; 2) There is a pre-trained score model represents the desired policy: $\forall k \in[\epsilon, K]: s_{\phi^*}(a^k,k)=\nabla\log p^k(a^k)$, which cannot be accessed; 3) $f_{\theta}$ satisfies the Lipschitz condition: there exists $L>0$ such that for all $k\in [\epsilon, K]$, and $x,y\in A$, we have $||f_{\theta}(x,k)-f_{\theta}(y,k)||_2\leq L||x-y||_2$. 
    
    The following training loss for the consistency policy is sufficient to replace the distillation loss:
    \begin{equation}
    \begin{aligned}
        L(\theta,\theta^-)=\alpha L_{CT}(\theta,\theta^-)-\beta\mathbb{E}[ Q(s, \hat{a}^{\epsilon})]+o(\Delta k)
    \end{aligned}
    \end{equation}
    
    \noindent where 
    \begin{equation}
    \begin{aligned}
        \alpha=&\mathbb{E}[1+Q(s,a)\frac{k_{m+1}^2}{\lambda(a^{k_{m+1}}-a)^2}]\\
        \beta=&\mathbb{E}[d(f_{\theta}(a+k_{m+1}z,k_{m+1}|s),f_{\theta^-}(a+k_{m+1}z,k_{m+1}|s))*\frac{k_{m+1}^2}{\lambda(a^{k_{m+1}}-a)^2})]\\
    \end{aligned}
    \end{equation}
\end{theorem}

\begin{proof}
    The Lemma 1 in Appendix A.3 of \cite{song2023consistency} provides a theoretical justification for the distillation loss and the consistency training loss. The expression of the score function can be further simplified to yield:
    \begin{equation}
        \label{eq:cm score}
        \nabla\log p^k(a^k)=-\mathbb{E}[\frac{a^k-a}{k^2}|a^k]
    \end{equation}
    where $k \in [\epsilon,K]$. 
    
    According to Eq. ({\ref{eq:cm score}}) and Eq. (\ref{eq:optimal policy}) claimed in Section \ref{sec: offline rl}, we can derive the following expression of the score function for the consistency policy:
    \begin{equation}
    \begin{aligned}
        \nabla\log \pi^k(a^k) = \nabla\log \mu_{k}(a^k)+\nabla \frac{1}{\lambda}Q(s,a^k)= -\mathbb{E}[\frac{a^k-a}{k^2}|a^k]+\nabla \frac{1}{\lambda}Q(s,a^k)
    \end{aligned}
    \end{equation}
    
    With Taylor expansion, we have:
    \begin{equation}
    \begin{aligned}
        L_{CD}&(\theta,\theta_{-};\phi^*)=\mathbb{E}[d(f_{\theta}(a+k_{m+1}z,k_{m+1}|s),f_{\theta^{-}}(\hat{a}_{\phi^*}^{k_m},k_m|s))]\\
        =&\mathbb{E}[d(f_{\theta}(a^{k_{m+1}},k_{m+1}|s),f_{\theta^-}(a^{k_{m+1}}+(k_{m+1}-k_m)k_{m+1}\nabla\log \pi_{k_{m+1}}(a^{k_{m+1}}),k_m|s))]\\
        =&\mathbb{E}[d(f_{\theta}(a^{k_{m+1}},k_{m+1}|s),f_{\theta^-}(a^{k_{m+1},k_{m+1}}|s)\\
        &+\partial_1f_{\theta^-}(a^{k_{m+1}},k_{m+1}|s)\nabla \pi_{k_{m+1}}(a^{k_{m+1}})+\partial_2f_{\theta^-}(a^{k_{m+1},k_{m+1}}|s)(k_m-k_{m+1}))+o(|k_{m+1}-k_m|)]\\
        =&\mathbb{E}[d(f_{\theta}(a^{k_{m+1}},k_{m+1}|s),f_{\theta^-}(a^{k_{m+1}},k_{m+1}|s))\\
        &+\partial_2d(f_{\theta}(a^{k_{m+1}},k_{m+1}|s),f_{\theta^-}(a^{k_{m+1}},k_{m+1}|s))[\partial_1f_{\theta^-}(a^{k_{m+1}},k_{m+1}|s)\nabla \pi_{k_{m+1}}(a^{k_{m+1}})\\
        &+\partial_2f_{\theta^-}(a^{k_{m+1}},k_{m+1}|s)(k_m-k_{m+1})]+o(|k_{m+1}-k_m|))]\\
        =&\mathbb{E}[d(f_{\theta}(a^{k_{m+1}},k_{m+1}|s),f_{\theta^-}(a^{k_{m+1}},k_{m+1}|s))]\\
        &+\mathbb{E}[\partial_2d(f_{\theta}(a^{k_{m+1}},k_{m+1}|s),f_{\theta^-}(a^{k_{m+1}},k_{m+1}|s))[\partial_1f_{\theta^-}(a^{k_{m+1}},k_{m+1}|s)(k_{m+1}-k_m)k_{m+1}\nabla \pi_{k_{m+1}}(a^{k_{m+1}})]]\\
        &+\mathbb{E}[\partial_2d(f_{\theta}(a^{k_{m+1}},k_{m+1}|s),f_{\theta^-}(a^{k_{m+1}},k_{m+1}|s))[\partial_2f_{\theta^-}(a^{k_{m+1}},k_{m+1}|s)(k_m-k_{m+1})]+o(|k_{m+1}-k_m|)]\\
        =& \mathcal A + \mathcal B
    \end{aligned}
    \end{equation}
    
    \noindent where
    \begin{equation}
    \begin{aligned}
        &\mathcal A=\mathbb{E}[d(f_{\theta}(a^{k_{m+1}},k_{m+1}|s),f_{\theta^-}(a^{k_{m+1}},k_{m+1}|s))]\\
        &+\mathbb{E}[\partial_2d(f_{\theta}(a^{k_{m+1}},k_{m+1}|s),f_{\theta^-}(a^{k_{m+1}},k_{m+1}|s))[\partial_1f_{\theta^-}(a^{k_{m+1}},k_{m+1}|s)(k_m-k_{m+1})k_{m+1}\mathbb{E}[\frac{a^{k_{m+1}}-a}{k^2}|a^{k_{m+1}}]]\\
        &+\mathbb{E}[\partial_2d(f_{\theta}(a^{k_{m+1}},k_{m+1}|s),f_{\theta^-}(a^{k_{m+1}},k_{m+1}|s))[\partial_2f_{\theta^-}(a^{k_{m+1}},k_{m+1}|s)(k_m-k_{m+1}))]]+\mathbb{E}[o(|k_{m+1}-k_m|)]\\
    \end{aligned}
    \end{equation}
    
    \noindent and
    \begin{equation}
    \begin{aligned}
        \label{eq:termB}
        &\mathcal B=\mathbb{E}[\partial_2d(f_{\theta}(a^{k_{m+1}},k_{m+1}|s),f_{\theta^-}(a^{k_{m+1}},k_{m+1}|s))[\partial_1f_{\theta^-}(a^{k_{m+1}},k_{m+1}|s)(k_{m+1}-k_m)k_{m+1}\frac{1}{\lambda}\nabla Q(s,a^{k_{m+1}})]]\\
    \end{aligned}
    \end{equation}
    
    The Theorem 2 in \cite{song2023consistency} has proofed that:
    \begin{equation}
    \begin{aligned}
        &\mathcal A=L_{CT}(\theta,\theta^-)+o(\Delta k)\\
    \end{aligned}
    \end{equation}
    
    We will then focus on the term $\mathcal B$. With Taylor expansion, we have:
    \begin{equation}
    \begin{aligned}
        Q(s,a)=Q(s,a^{k_{m+1}})+\partial_1Q(s,a^{k_{m+1}})(a-a^{k_{m+1}})\\
    \end{aligned}
    \end{equation}

    Thus we have:
    \begin{equation}
    \begin{aligned}
    \label{eq:deltaQ}
        \partial_1Q(s,a^{k_{m+1}})=\frac{Q(s,a)-Q(s,a^{k_{m+1}})}{a-a^{k_{m+1}}}\\
    \end{aligned}
    \end{equation}
    
    According to Eq. (\ref{eq:termB}) and Eq. (\ref{eq:deltaQ}), we can further derive the following relation:
    \begin{equation}
    \begin{aligned}
        \mathcal B=&\frac{1}{\lambda}\mathbb{E}[\partial_2d(f_{\theta}(a^{k_{m+1}},k_{m+1}|s),f_{\theta^-}(a^{k_{m+1}},k_{m+1}|s))[\partial_1f_{\theta^-}(a^{k_{m+1}},k_{m+1}|s)(k_{m+1}-k_m)k_{m+1}\frac{Q(s,a)-Q(s,a^{k_{m+1}})}{a-a^{k_{m+1}}}]]\\
        \overset{(i)}{=}&\frac{1}{\lambda}\mathbb{E}[\partial_2d(f_{\theta}(a^{k_{m+1}},k_{m+1}|s),f_{\theta^-}(a^{k_{m+1}},k_{m+1}|s))[\partial_1f_{\theta^-}(a^{k_{m+1}},k_{m+1}|s)(k_{m+1}-k_m)\frac{Q(s,a^{k_{m+1}})-Q(s,a)}{z}]]\\
        =&-\frac{1}{\lambda}\mathbb{E}[\partial_2d(f_{\theta}(a^{k_{m+1}},k_{m+1}),f_{\theta^-}(a^{k_{m+1}},k_{m+1}|s))[\partial_1f_{\theta^-}(a^{k_{m+1}},k_{m+1}|s)(k_m-k_{m+1})z(Q(s,a^{k_{m+1}})-Q(s,a))\frac{k_{m+1}^2}{(a^{k_{m+1}}-a)^2}]]\\
    \end{aligned}
    \end{equation}
    
    \noindent where $(i)$ is due to $a^{k_{m+1}}\sim \mathcal{N}(a,k_{m+1}^2I)$. Then we use the Taylor expansion in the reverse direction to obtain:        
    \begin{equation}
    \begin{aligned}
        \mathcal B=&\frac{1}{\lambda}\mathbb{E}[(d(f_{\theta}(a^{k_{m+1}},k_{m+1}|s),f_{\theta^-}(a^{k_{m+1}},k_{m+1}|s))\\
        &-d(f_{\theta}(a^{k_{m+1}},k_{m+1}|s),f_{\theta^-}(a^{k_{m+1}}+(k_m-k_{m+1})z,k_{m+1}|s)))(Q(s,a^{k_{m+1}})-Q(s,a))\frac{k_{m+1}^2}{(a^{k_{m+1}}-a)^2}]\\
        =&\frac{1}{\lambda}\mathbb{E}[(d(f_{\theta}(a^{k_{m+1}},k_{m+1}|s),f_{\theta^-}(a^{k_{m+1}},k_{m+1}|s))-d(f_{\theta}(a^{k_{m+1}},k_{m+1}),f_{\theta^-}(a+k_mz,k_{m+1}|s)))(Q(s,a^{k_{m+1}})-Q(s,a))\frac{k_{m+1}^2}{(a^{k_{m+1}}-a)^2}]\\
        \overset{(ii)}{\approx}&\frac{1}{\lambda}\mathbb{E}[(d(f_{\theta}(a^{k_{m+1}},k_{m+1}|s),f_{\theta^-}(a^{k_{m+1}},k_{m+1}|s))-d(f_{\theta}(a^{k_{m+1}},k_{m+1}),f_{\theta^-}(a^{k_m},k_{m}|s)))(Q(s,a^{k_{m+1}})-Q(s,a))\frac{k_{m+1}^2}{(a^{k_{m+1}}-a)^2}]\\
    \end{aligned}
    \end{equation}
    
    \noindent where $(ii)$ takes the approximation of $k_m\approx k_{m+1}$, which will result in an error bounded by $L||k_{m+1}-k_m||_2$. 
    
    Following \cite{lu2023contrastive}, we could have:
    \begin{equation}
    \begin{aligned}
        Q(s,a^{k_m})&=\log \mathbb{E}_{p_{data}(a|a^{k_m})}[e^{\beta Q(s,a)}]\\
        &=\log\int \mu(a)e^{Q(s,a)}da\\
        &=\log\int \pi(a)da\\
        &=0
    \end{aligned}
    \end{equation}
    
    Then we have:
    \begin{equation}
    \begin{aligned}
        \mathcal B=-\frac{1}{\lambda}\mathbb{E}[(d(f_{\theta}(a^{k_{m+1}},k_{m+1}|s),f_{\theta^-}(a^{k_{m+1}},k_{m+1}|s))-d(f_{\theta}(a^{k_{m+1}},k_{m+1}|s),f_{\theta^-}(a^{k_m},k_{m}|s)))Q(s,a)\frac{k_{m+1}^2}{(a^{k_{m+1}}-a)^2}]\\
    \end{aligned}
    \end{equation}

    Thus, we can derive that: 
    \begin{equation}
    \begin{aligned}
        \mathcal A + \mathcal B =&\mathbb{E}[d(f_{\theta}(a^{k_{m+1}},k_{m+1}|s),f_{\theta^-}(a^{k_m},k_{m}|s))(1+Q(s,a)\frac{k_{m+1}^2}{\lambda(a^{k_{m+1}}-a)^2})\\
        &-d(f_{\theta}(a^{k_{m+1}},k_{m+1}|s),f_{\theta^-}(a^{k_{m+1}},k_{m+1}|s)) Q(s,a)\frac{k_{m+1}^2}{\lambda(a^{k_{m+1}}-a)^2})] + o(\Delta k)\\
        \overset{(iii)}{\approx}&\mathbb{E}[d(f_{\theta}(a^{k_{m+1}},k_{m+1}|s),f_{\theta^-}(a^{k_m},k_{m}|s))(1+Q(s,a)\frac{k_{m+1}^2}{\lambda(a^{k_{m+1}}-a)^2})\\
        &-d(f_{\theta}(a^{k_{m+1}},k_{m+1}),f_{\theta^-}(a^{k_{m+1}},k_{m+1}|s)) Q(s,\hat{a}^{\epsilon})\frac{k_{m+1}^2}{\lambda(a^{k_{m+1}}-a)^2})] + o(\Delta k)\\
    \end{aligned}
    \end{equation}
    where $(iii)$ takes the approximation of $a\approx\hat{a}^{\epsilon}$

    For the consistency model training, we define the hyperparameter:
    \begin{equation}
    \begin{aligned}
        \alpha=&\mathbb{E}[1+Q(s,a)\frac{k_{m+1}^2}{\lambda(a^{k_{m+1}}-a)^2}]
    \end{aligned}
    \end{equation}
    
    \noindent and
    \begin{equation}
    \begin{aligned}
        \beta&=\mathbb{E}[d(f_{\theta}(a^{k_{m+1}},k_{m+1}|s),f_{\theta^-}(a^{k_{m+1}},k_{m+1}|s))\frac{k_{m+1}^2}{\lambda(a^{k_{m+1}}-a)^2})]\\
    \end{aligned}
    \end{equation}

    Then we have:
    \begin{equation}
    \begin{aligned}
        L_{CD}(\theta,\theta^-;\phi^*)&=\mathcal A + \mathcal B=\alpha L_{CT}(\theta,\theta^-) -\beta\mathbb{E}[Q(s,\hat{a}^{\epsilon})] + o(\Delta k)
    \end{aligned}
    \end{equation}
    
    Note that the $a$ for calculating $L_{CT}(\theta,\theta^-)$ is sampled from the dataset, and $\hat{a}^{\epsilon}$ for the Q-function guidance term is generated from the consistency policy. The above proves the formulation of the final optimization objective.
\end{proof}

\subsection{The Proof of Reconstruction Loss}
\label{apx:recon loss}
% \begin{theorem}
    Let $\Delta k=\max_{m\in [1, M-1]}\{|k_{m+1}-k_m|\}$. We have the following assumptions:
    \begin{itemize}
        \item Assume we have a pre-trained score function represents the desired policy: $\forall k \in[\epsilon, K]: s_{\phi^*}(a^k,k)=\nabla\log p^k(a^k)$, which cannot be accessed. And let $f(a^k,k;\phi^*|s)$ be the ground truth conditional consistency policy of this score model.
    \end{itemize}
    Then, if $L_{RC}(\theta,\theta^{-})=0$, we have:
    \begin{center}
        $L_{CT}(\theta,\theta^{-})=L_{RC}(\theta,\theta^{-})$
    \end{center}
% \end{theorem}

\begin{proof}
    From $L_{RC}=0$, we have:
    \begin{equation}
        \mathbb{E}[d(f_{\theta}(a+k_{m+1}z,k_{m+1}|s),a)]  \equiv 0
    \end{equation}
    
    \noindent where the expectation is taken with respect to $a\sim p_{data}$, $m\sim \mathcal{U}[1, M-1]$, $z\sim \mathcal{N}(0,1)$. 
    
    Considering that $p(a)>0$ for any $a \in p_{data}$, it entails:
    \begin{equation}
        d(f_{\theta}(a+k_{m}z,k_{m}|s),a) \equiv 0
    \end{equation}
    
    \noindent for $m \in [1, M-1]$.

    With respect to $d(x,y)=0 \Leftrightarrow x-y=0$, we have:
    \begin{equation}
    \begin{aligned}
        f_{\theta}(a+k_{m+1}z,k_{m+1}|s) &\equiv a \\
        &\equiv f(a+k_{m+1}z,k_{m+1};\phi^*|s)\\
        &= f(a+k_{m}z,k_{m};\phi^*|s)
    \end{aligned}
    \end{equation}

    Then we have:
    \begin{equation}
    \begin{aligned}
        L_{CT}(\theta,\theta^{-}) &= \mathbb{E}[d(f_{\theta}(a+k_{m+1}z,k_{m+1}|s),f_{\theta^{-}}(a+k_{m}z,k_{m}|s))]\\
        &= \mathbb{E}[d(f(a+k_{m}z,k_{m};\phi^*|s),f_{\theta^{-}}(a+k_{m}z,k_{m}|s))]\\
        &= \mathbb{E}[d(f_{\theta^{-}}(a+k_{m}z,k_{m}|s), a)]\\
        &= L_{RC}(\theta,\theta^{-})
    \end{aligned}
    \end{equation}
\end{proof}

\section{Consistency Policy for Offline RL and Online RL}
\label{apx: policy}
We provide an illustrative description of how CPQL works for both offline and online RL tasks, as shown in Figure \ref{fig:rl goal}. One one hand, previously used unimodal policy, such as Gaussian policy $\pi_g$, has to make the trade-off between better performance and diverse behaviors. The Gaussian policy learns with extra exploration policy $\pi_e$ for online tasks and struggles with constraints from behavior policy $\pi_b$ for offline tasks. On the other hand, the consistency policy $\pi_{cp}$ has a powerful ability to represent complex data distribution. The policy $\mu$ refers to the behavior policy under offline scenario while representing the policy $\pi_r$ for collecting the data in replay buffer $\mathcal{D}_{r}$ under online scenario. This consistency policy concurrently handles greedy behavior and diverse exploratory behaviors or policy constraints in a seamless manner.

\begin{figure*}[t]
    \centering
    \includegraphics[width=16cm]{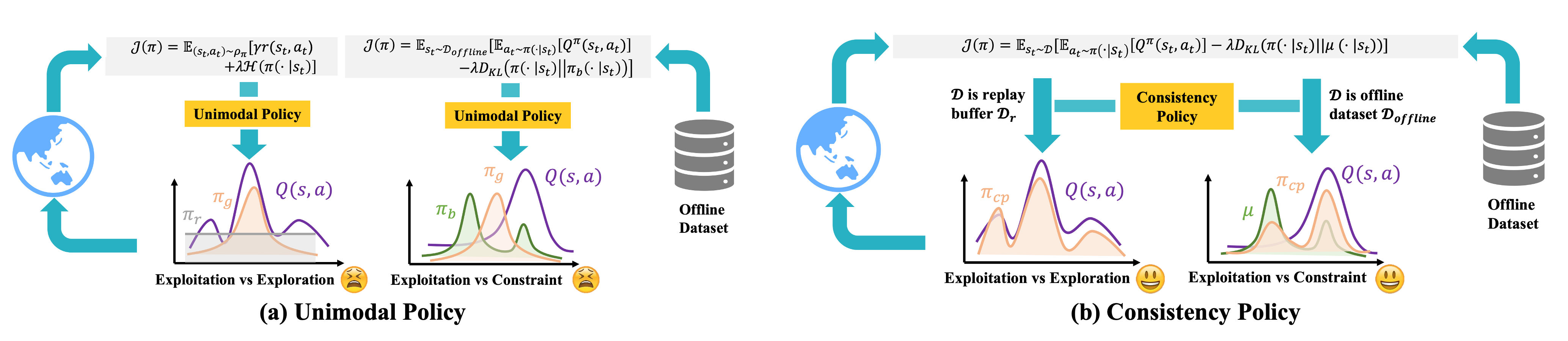}
    \caption{Illustration of online and offline RL with (a) previously used unimodal policy and (b) our proposed consistency policy. 
    }
\label{fig:rl goal}
\end{figure*}

\section{Experiment Setup}
\label{apx: exp setup}
In this section, we provide the setup for all experiments tested in the paper and the source code we used for reporting the results of other methods, such as TD3, SAC, PPO, DIPO, and Diffusion-QL. 

The consistency policy is based on a multilayer perceptron (MLP) that takes the state as input and provides the corresponding action as output. For the policy networks, we establish the backbone network architecture to be a 3-layer MLP (hidden size 256) with the Mish function serving as the activation function. The value networks also consist of a 3-layer MLP with Mish function, while additionally using LayNorm between each layer. 

\begin{table}[ht]
  \setlength{\abovecaptionskip}{0.2cm}
  \centering
  \caption{The hyperparameters of all selected tasks for CPQL and CPIQL.}
  \scalebox{1}{
  \begin{tabular}{lcccc}             \\
    \toprule
    Tasks                         & $\alpha$ & $\eta$ & $\tau$ & learning rate   \\
    \bottomrule
    \toprule
    All online tasks              & 0.05     & 1       &-      & $3*10^{-4}$\\
    \bottomrule
    \toprule
    halfcheetah-medium-v2         & 1        & 2.0     & 0.7   & $3*10^{-4}$\\
    hopper-medium-v2              & 1        & 2.0     & 0.6   & $3*10^{-4}$\\ 
    walker2d-medium-v2            & 1        & 1.0     & 0.6   & $3*10^{-4}$\\ 
    halfcheetah-medium-replay-v2  & 1        & 3.0     & 0.7   & $3*10^{-4}$\\ 
    hopper-medium-replay-v2       & 1        & 1.0     & 0.6   & $3*10^{-4}$\\ 
    walker2d-medium-replay-v2     & 1        & 1.0     & 0.6   & $3*10^{-4}$\\ 
    halfcheetah-medium-expert-v2  & 1        & 1.0     & 0.7   & $3*10^{-4}$\\ 
    hopper-medium-expert-v2       & 1        & 0.5     & 0.6   & $3*10^{-4}$\\ 
    walker2d-medium-expert-v2     & 1        & 1.0     & 0.6   & $3*10^{-4}$\\ 
    \bottomrule
    \toprule
    pen-human-v1                  & 1        & 0.15    & 0.7   & $3*10^{-5}$\\ 
    pen-cloned-v1                 & 1        & 0.1     & 0.7   & $3*10^{-5}$\\  
    \bottomrule
  \end{tabular}
  }
  \label{tab:hyperparameters}
\end{table}

For the consistency policy, we set $k \in [0.002, 80.0]$ and the number of sub-intervals $M=40$. Following Karras diffusion model \cite{karras2022elucidating}, the sub-interval boundaries are determined with the formula $k_i=(\epsilon^{\frac{1}{\rho}}+\frac{i-1}{M-1}(T^{\frac{1}{\rho}}-\epsilon^{\frac{1}{\rho}}))^{\rho}$, where $\rho=7$. And we use the Euler solver as the ODE solver. For the offline scenario, we train for 1000 epochs (2000 for Locomotion tasks), each consisting of 1000 gradient steps with batch size 256 on \textsf{D4RL} tasks. For the online scenario, we train for 1M iterations on \textsf{Gym MuJoCo} tasks and 500k iterations on \textsf{dm\_control} tasks. For learning rate, we set $3*10^{-5}$ for Adroit tasks of \textsf{D4RL} tasks and $3*10^{-5}$ for all other tasks. And we use a fixed learning rate, $3*10^{-4}$, for Q-function networks. Considering the $\alpha$ and $\eta$, for online tasks such as \textsf{Gym MuJoCo} tasks and \textsf{dm\_control} tasks, we set $\alpha=0.05$ and $\eta=1$. For offline tasks such as \textsf{D4RL} tasks, we set $\alpha=1$ and give the optimal $\eta$ for different tasks. Typically for CPIQL on offline tasks, we provide the hyperparameter $\tau$ for the expectile regression function for different tasks.

The above-mentioned hyperparameters we used for different tasks are shown in Table \ref{tab:hyperparameters}. And the Python codes for reporting the results of TD3, SAC, PPO, DIPO, EDP, and Diffusion-QL are the following:
\begin{itemize}
    \item TD3: \url{https: //github.com/sfujim/TD3};
    \item SAC: \url{https://github.com/toshikwa/soft-actor-critic.pytorch};
    \item PPO: \url{https://github.com/ikostrikov/pytorch-a2c-ppo-acktr-gail};
    \item DIPO: \url{https://github.com/BellmanTimeHut/DIPO};
    \item EDP: \url{https://github.com/sail-sg/edp};
    \item Diffusion-QL: \url{https://github.com/Zhendong-Wang/Diffusion-Policies-for-Offline-RL}.
\end{itemize}

\begin{table}[ht]
  \setlength{\abovecaptionskip}{0.2cm}
  \centering
  \caption{The performance of Diffusion-QL, CPQL and CPIQL on \textsf{D4RL} Locomotion and Adroit tasks. For CPQL and CPIQL, we report the mean and standard deviation of normalized scores across five random seeds. The results of Diffusion-QL(best) are from the paper of Diffusion-QL \cite{wang2022diffusion}, and the results of Diffusion-QL(final) are from the paper of QGPO \cite{lu2023contrastive}.}
  \scalebox{1}{
  \begin{tabular}{lcc|cc|cc}
    % \multicolumn{2}{c}{Part}                   \\
    \toprule
    Dataset                      &Diffusion-QL(best) &Diffusion-QL(final) &CPIQL(best)              &CPIQL(final)    &CPQL(best)               &CPQL(final)      \\
    \bottomrule
    \toprule
    halfcheetah-medium-v2        &51.1               &50.6                &55.3 $\pm$ 0.4           &54.6 $\pm$ 1.0  &\textbf{57.9 $\pm$ 0.3}  &56.9 $\pm$ 0.9       \\
    hopper-medium-v2             &90.5               &82.4                &101.5 $\pm$ 2.1          &99.7 $\pm$ 2.0  &\textbf{102.1 $\pm$ 2.4} &99.9 $\pm$ 4.5      \\
    walker2d-medium-v2           &87.0               &85.1                &88.4 $\pm$ 0.5           &86.2 $\pm$ 0.6  &\textbf{90.5 $\pm$ 2.2}  &82.1 $\pm$ 2.4       \\
    halfcheetah-medium-replay-v2 &47.8               &47.5                &\textbf{49.8 $\pm$ 1.0}  &48.0 $\pm$ 1.4  &48.1 $\pm$ 0.7           &46.6 $\pm$ 0.8       \\
    hopper-medium-replay-v2      &101.3              &100.7               &\textbf{101.7 $\pm$ 0.6} &100.6 $\pm$ 1.5 &\textbf{101.7 $\pm$ 0.9} &97.7 $\pm$ 4.6     \\
    walker2d-medium-replay-v2    &\textbf{95.5}      &94.3                &95.0 $\pm$ 0.6           &91.8 $\pm$ 2.8  &94.4 $\pm$ 1.3           &93.6 $\pm$ 5.6      \\
    halfcheetah-medium-expert-v2 &96.8               &96.1                &90.2 $\pm$ 1.3           &81.0 $\pm$ 1.7  &\textbf{98.8 $\pm$ 0.4}  &97.8 $\pm$ 0.5      \\
    hopper-medium-expert-v2      &111.1              &110.7               &113.4 $\pm$ 0.8          &110.6 $\pm$ 1.4 &\textbf{114.2 $\pm$ 0.4} &110.4 $\pm$ 3.2      \\
    walker2d-medium-expert-v2    &110.1              &109.7               &112.3 $\pm$ 0.3          &110.9 $\pm$ 0.2 &\textbf{111.5 $\pm$ 0.1} &110.9 $\pm$ 0.1     \\
    \midrule
    \rowcolor{gray!20} Average   &87.9               &86.3                &89.7                     &87.0            &\textbf{91.0}            &88.4               \\
    \bottomrule
    \toprule
    pen-human-v1                 &72.8               &-                   &58.2 $\pm$ 2.7           &48.0 $\pm$ 8.5  &\textbf{89.3 $\pm$ 6.9}  &56.7 $\pm$ 4.9           \\
    pen-cloned-v1                &57.3               &-                   &77.4 $\pm$ 6.9           &57.4 $\pm$ 5.0  &\textbf{83.3 $\pm$ 6.1}  &65.3 $\pm$ 2.5           \\
    \midrule
    \rowcolor{gray!20} Average   &65.1               &-                   &67.8                     &52.7            &\textbf{86.3}            &61.0                  \\
    \bottomrule
  \end{tabular}
  }
  \label{tab:offline scores best final}
\end{table}

\begin{table*}[ht]
  \centering
  \caption{The original training IPS and sampling SPS of CPQL and other methods on \textsf{Gym MuJoCo} tasks using Nvidia 2080Ti. }
  \scalebox{1}{
  \begin{tabular}{lcccccccc|c}             \\
    \toprule
    -   & Diffusion-QL (PyTorch) & DIPO (PyTorch) & CPQL (PyTorch) & Diffusion-QL (Jax) & EDP (Jax) & CPQL (Jax)  \\
    \bottomrule
    \toprule
    IPS & 4.7                    & 8.2            & 68.0           & 22.3               & 116.2     & 325.2       \\
    SPS & 15.1                   & 17.0           & 680.2          & 123.7              & 411.9     & 1069.7      \\
    \bottomrule
  \end{tabular}
  }
  \label{tab:ips sps}
\end{table*}

\section{More Experiment Results}
\label{apx:more d4rl}
In this section, we provide more experiment results for \textsf{D4RL} tasks under the offline scenario, as shown in Table \ref{tab:offline scores best final}. Different from Table 3 in the paper, we here provide 2 columns (best and final) for CPQL and CPIQL, where "best" represents the best performance during training with online evaluation, and "final" refers to the performance at the end of training. In most cases, the final performance can reach the best performance unless the training is unstable, leading to performance collapse in the later training process. In this paper, we choose the "best" scores for method comparison.

Also, we provide the IPS for training and SPS for inference for chosen methods are provided in Table \ref{tab:ips sps}.

\section{Training Curves for Offline and Online Reinforcement Learning}
\label{apx:train curve}
In this section, we provide curves of the training process for each selected task mentioned in this paper. We plot the mean and standard deviation of results across five random seeds as shown in Figure \ref{fig:train curve dm control}, Figure \ref{fig:train curve mujoco} and Figure \ref{fig:train curve d4rl}.

\begin{figure*}[htbp]
    \centering 
    \includegraphics[width=14cm]{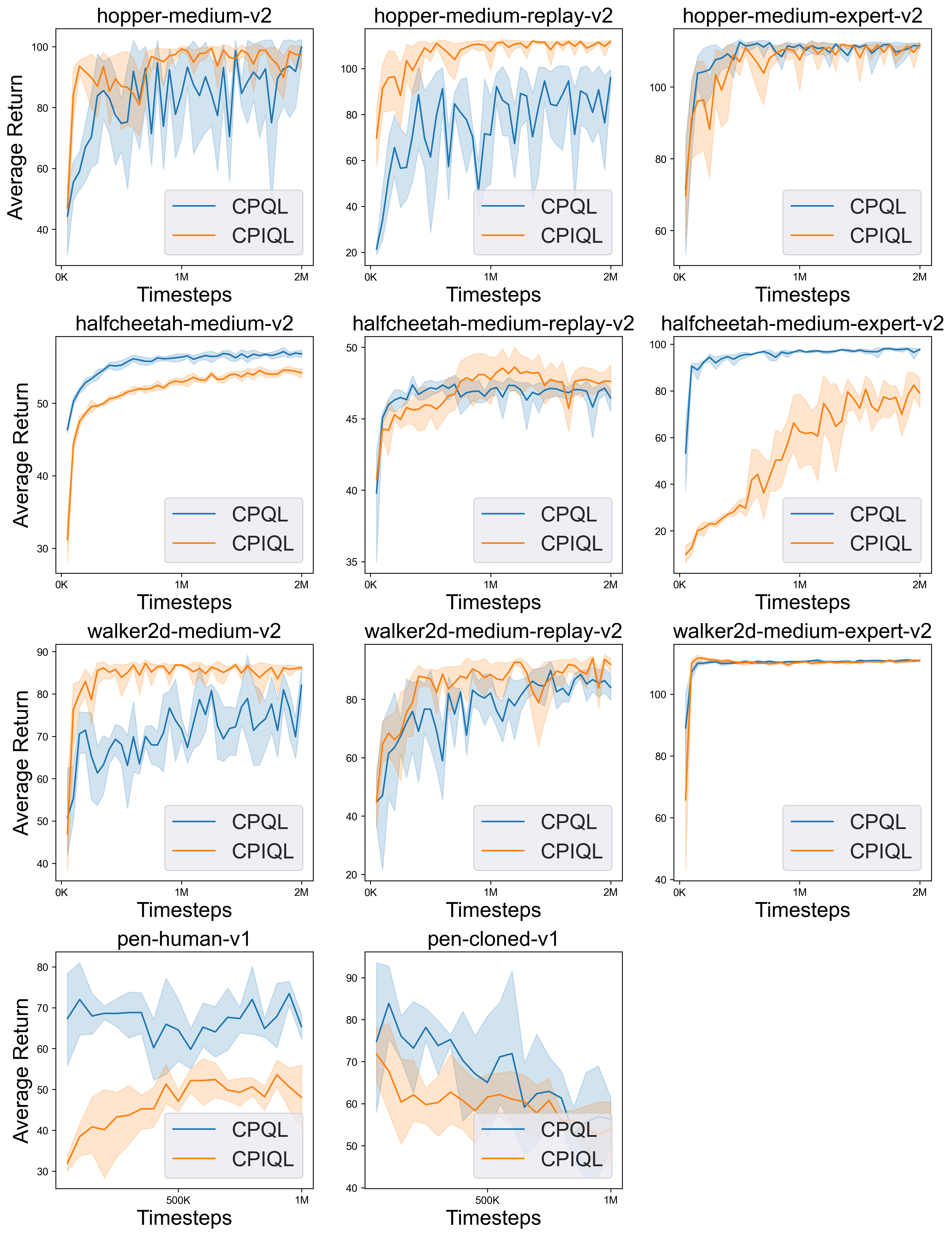}
    \caption{Training curves for \textsf{D4RL} tasks under the offline scenario.}
    \label{fig:train curve d4rl}
\end{figure*}

\begin{figure*}[htbp]
    \centering 
    \includegraphics[width=14cm]{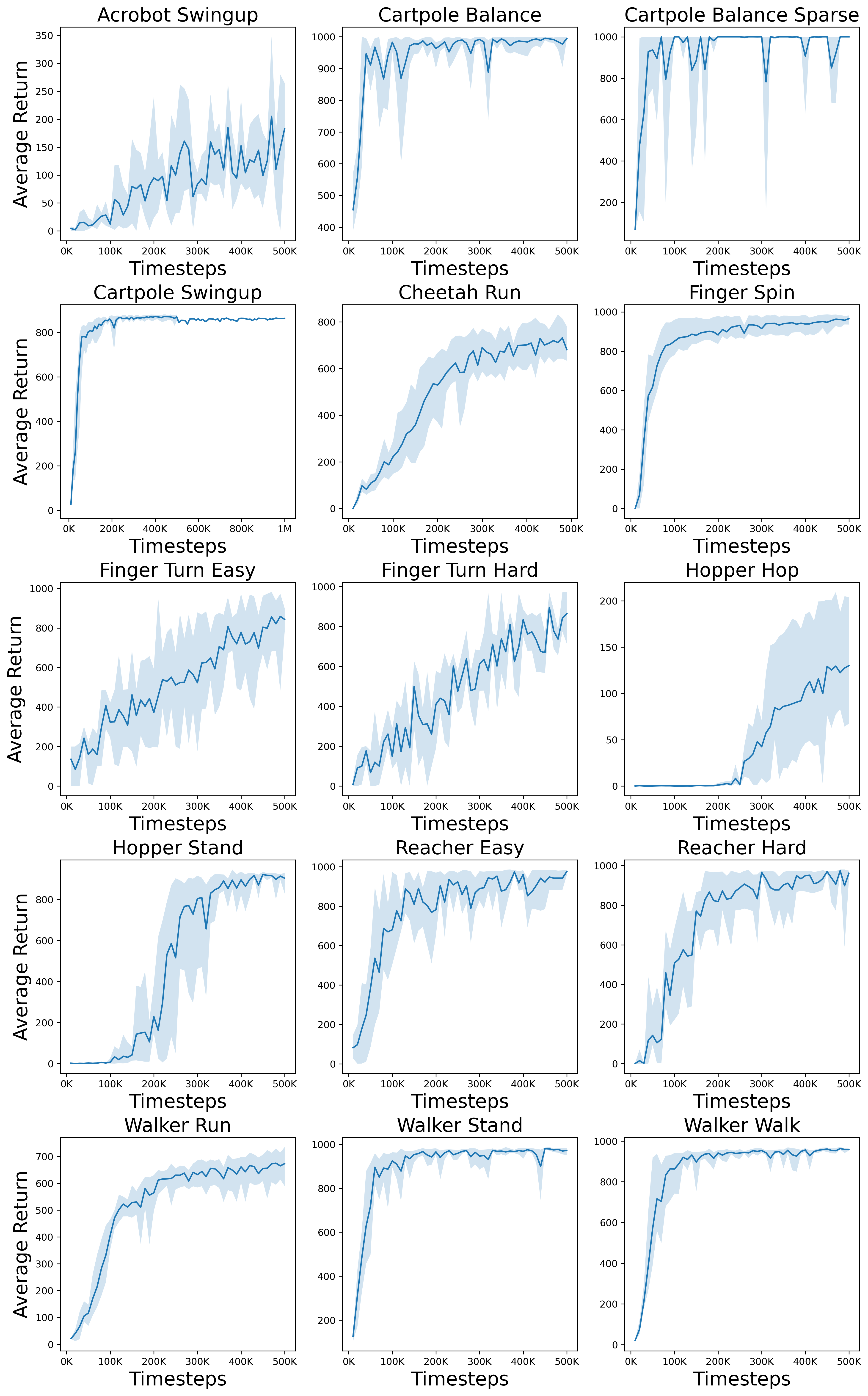}
    \caption{Training curves for \textsf{dm\_control} tasks under the online scenario}
    \label{fig:train curve dm control}
\end{figure*}

\begin{figure*}[htbp]
    \centering 
    \includegraphics[width=14cm]{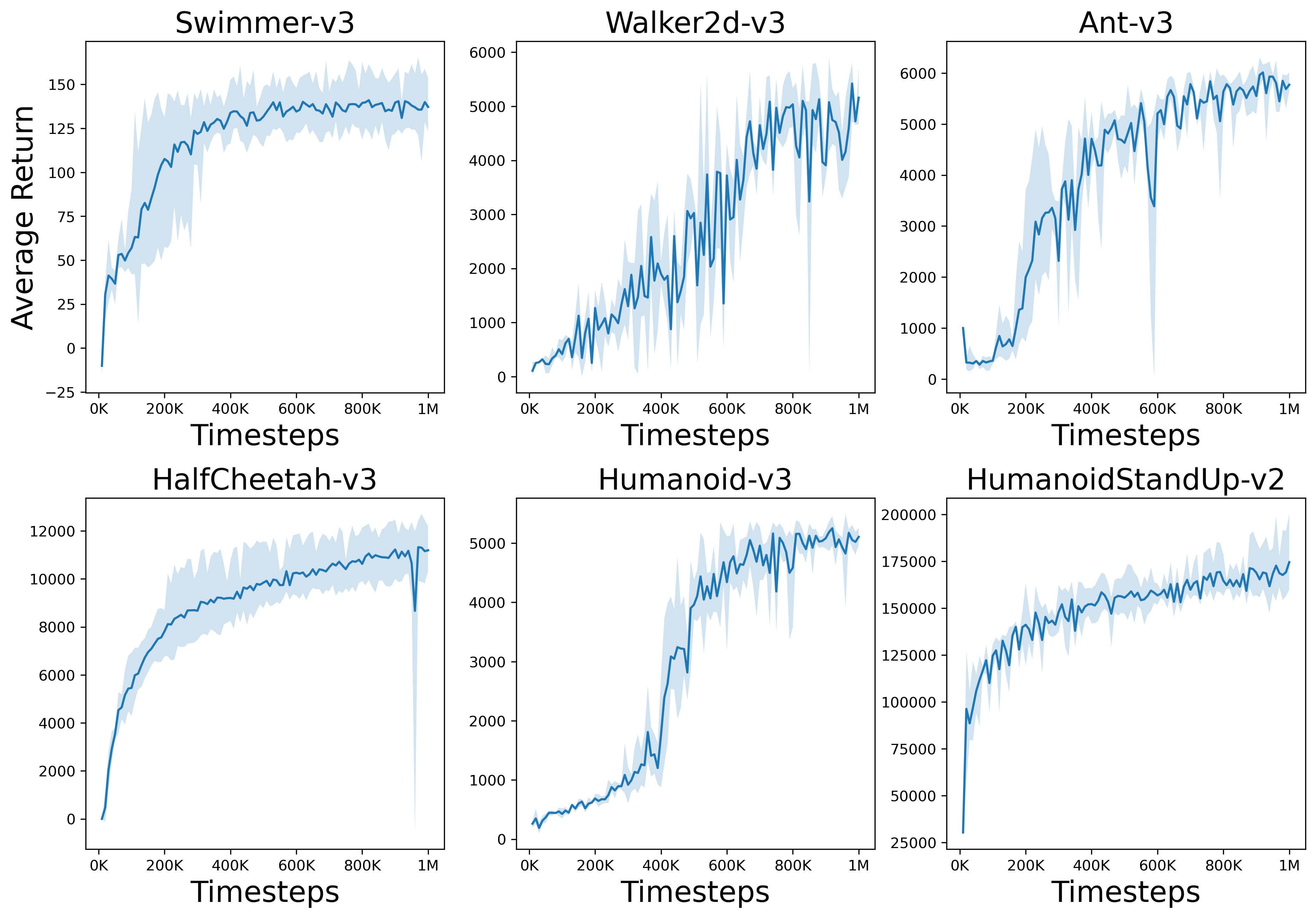}
    \caption{Training curves for \textsf{Gym MuJoCo} tasks under the online scenario.}
    \label{fig:train curve mujoco}
\end{figure*}

%%%%%%%%%%%%%%%%%%%%%%%%%%%%%%%%%%%%%%%%%%%%%%%%%%%%%%%%%%%%%%%%%%%%%%%%

\end{document}